\documentclass[12pt]{article}

\usepackage{mathrsfs}
\usepackage{natbib,graphicx,setspace,lscape,longtable}
\usepackage{mathrsfs,amsmath,amsthm,amssymb,color}
\usepackage{natbib,epsfig,graphicx,pdfpages}
\usepackage{rotating}
\usepackage{float}
\usepackage{bm}
\usepackage{ulem}
\usepackage[ruled]{algorithm2e}
\usepackage{CJK}
\usepackage{natbib}
\usepackage{booktabs}
\usepackage{bm}
\usepackage{subfigure}
\usepackage{makecell}
\usepackage{appendix}
\usepackage{hyperref}
\usepackage{multirow}
\usepackage{threeparttable}
\usepackage{xr}
\hypersetup{
colorlinks=true,
linkcolor=blue,
citecolor=blue,
urlcolor=blue}

\bibpunct{(}{)}{;}{a}{,}{,}

\newtheorem{theorem}{Theorem}
\newtheorem{lemma}{Lemma}
\newtheorem{proposition}{Proposition}
\newtheorem{corollary}{Corollary}

\newcommand{\csection}[1]
{\begin{center}
\stepcounter{section}
{\bf\large\arabic{section}. #1}
\end{center}
}

\newcommand{\scsection}[1]
{\begin{center}
{\bf\large #1}
\end{center}
}

\newcommand{\csubsection}[1]{
\begin{center}
\stepcounter{subsection}
{\it\arabic{section}.\arabic{subsection}. #1}
\end{center}
}

\newcommand{\scsubsection}[1]{
\begin{center}
\stepcounter{subsection}
{\it #1}
\end{center}
}

\def\beq{\begin{equation}}
\def\eeq{\end{equation}}
\def\beqr{\begin{eqnarray}}
\def\eeqr{\end{eqnarray}}
\def\beqrs{\begin{eqnarray*}}
\def\eeqrs{\end{eqnarray*}}
\def\bet{\begin{theorem}}
\def\eet{\end{theorem}}
\def\bel{\begin{lemma}}
\def\eel{\end{lemma}}
\def\bep{\begin{proposition}}
\def\eep{\end{proposition}}
\def\bg{\begin{figure}[tbph]\begin{center}}
\def\eg{\end{center}\end{figure}}

\def\bc{\begin{center}}
\def\ec{\end{center}}

\def\epi{\varepsilon}
\def\wh{\widehat}

\def\mR{\mathbb{R}}
\def\mS{\mathcal S}

\def\mL{\mathcal L}
\def\thetap{\wh \theta_{\operatorname{stage,0}} }
\def\thetaos{\wh\theta_{\operatorname{stage,1}}}
\def\thetatwo{\wh\theta_{\operatorname{stage,2}}}
\def\thetaK{\wh\theta_{\operatorname{stage,K}}}

\def\thetaspt{\wh\theta_{\operatorname{stage,t}}}
\def\thetasptt{\wh\theta_{\operatorname{stage,t+1}}}
\def\thetaktwo{\wh\theta_{\operatorname{stage,K-2}}}
\def\thetakone{\wh\theta_{\operatorname{stage,K-1}}}
\def\thetakmone{\wh\theta_{\operatorname{stage,k-1}}}
\def\thetakm{\wh\theta_{\operatorname{stage,k}}}

\def\1{\mbox{\boldmath $1$}}

\def\mM{\mathcal M}

\def\mO{\mathcal O}

\def\cov{\mbox{cov}}

\def\argmin{\mbox{argmin}}

\newcommand{\mE}{{\mathcal E}}
\numberwithin{equation}{section}

\def\HNT{H^{(t)}}
\def\HNTT{H^{(t+1)}}
\def\st{ \wh \theta^{(t+1)} - \wh \theta^{(t)}}
\def\szero{ \thetaos - \thetap }
\def\yzero{\big\{\dot{\mL} (\thetaos) - \dot{\mL} (\thetap)\big\}}

\def\yt{\big\{\dot{\mL} (\wh \theta^{(t+1)}) - \dot{\mL} (\wh \theta^{(t)})\big\}}
\def\Vt{V^{(t)}}
\def\rhot{\rho^{(t)} }

\def\mSm{\mS_{(m)}}
\def\thetam{\wh \theta_{(m)}}

\def\hesmp{ \ddot{\mL}_{(m)}(\thetam) }
\def\hesmtrue{ \ddot{\mL}_{(m)}(\theta_0) }
\def\hestrue{ \ddot{\mL}(\theta_0) }
\def\hesint{ \ddot{\mL}(\thetap) }
\def\hestheta{ \ddot{\mL}(\theta) }
\def\hesmtheta{ \ddot{\mL}_{(m)}(\theta) }

\def\thetaglo{\wh \theta_{\operatorname{ge}} }
\def\thetat{\wh \theta^{(t)} }
\def\thetanr{\wh \theta_{\operatorname{nr,1}}}

\def\Hint{\big\{\ddot{\mL}(\thetap) \big\}^{-1}}
\def\Htrue{\big\{\ddot{\mL}(\theta_0) \big\}^{-1}}

\def\Havgtrue{ M^{-1} \sum_{m=1}^M \big\{ \ddot{\mL}_{(m)} (\theta_0) \big\}^{-1} }

\begin{document}
\begin{center}
{\bf\Large Quasi-Newton Updating for Large-Scale Distributed Learning}\\
\bigskip
Shuyuan Wu$^1$, Danyang Huang$^{2,3}$, and Hansheng Wang$^4$

{\it $^1$School of Statistics and Management, Shanghai University of Finance and Economics, Shanghai, China}\\
{\it $^2$Center for Applied Statistics, Renmin University of China, Beijing, China}\\
{\it $^3$School of Statistics, Renmin University of China, Beijing, China}\\
{\it $^4$Guanghua School of Management, Peking University, Beijing, China}
\end{center}
\begin{singlespace}
\begin{abstract}

Distributed computing is critically important for modern statistical analysis. Herein, we develop a distributed quasi-Newton (DQN) framework with excellent statistical, computation, and communication efficiency. In the DQN method, no Hessian matrix inversion or communication is needed. This considerably reduces the computation and communication complexity of the proposed method. Notably, related existing methods only analyze numerical convergence and require a diverging number of iterations to converge. However, we investigate the statistical properties of the DQN method and theoretically demonstrate that the resulting estimator is statistically efficient over a small number of iterations under mild conditions. Extensive numerical analyses demonstrate the finite sample performance.\\


  \begin{footnotetext}[1]
{
Correspondence: Danyang Huang, Center
for Applied Statistics, School of Statistics,
Renmin University of China, 59
Zhongguancun Street, Beijing 100872, China.
Email: dyhuang@ruc.edu.cn.
}
\end{footnotetext}

\end{abstract}
\end{singlespace}

\newpage

\csection{INTRODUCTION}

Modern statistical analysis often involves massive datasets \citep{gopal2013distributed}. In several cases, such datasets are too large to be efficiently handled by a single computer. Instead, they have to be divided and then processed on a distributed computer system, which consists of a large number of computers \citep{zhang2012communication}. Among all such computers, one often serves as the central computer, while the rest serve as worker computers. In this scenario, the central computer should be connected with all worker computers to construct a distributed computing system. Thus, approaches for the realization of efficient statistical learning on such distributed computing systems have received considerable interest from the research community \citep{mcdonald2009efficient,jordan2019communication,tang2020distributed,hector2020doubly,hector2021distributed}.

Here, we consider a standard statistical learning problem with a total of $N$ observations, where $N$ is assumed to be very large. For each observation $i$, we collect a response variable $Y_i\in\mR$ and corresponding feature vector $X_i\in\mR^p$. The objective is to accurately estimate an unknown parameter $\theta_0\in\mR^p$ by minimizing an appropriately defined empirical loss function (e.g., negative log-likelihood function), denoted by $\mL( \theta) = \sum_{i=1}^N \ell (X_i,Y_i;\theta) $, where $\ell (X_i,Y_i;\theta)$ is the loss function defined on the $i$-th sample.
Under a traditional setup with a small sample size $N$, this optimization problem can be easily solved using, for example, the standard Newton--Raphson algorithm. Specifically, let $\wh \theta^{(0)}$ be an appropriate initial estimator of $\theta_0$. Next, let $\wh \theta^{(t)}$ be the estimator obtained in the $t$-th iteration. Then, the ($t$+1)-th step estimator can be obtained as follows:
\beq
\label{eq:Newton-Raphson}
\wh \theta^{(t+1)} = \wh \theta^{(t)} - \alpha_t \big\{\ddot{\mL}\big(\wh \theta^{(t)} \big)\big\}^{-1} \dot{\mL}\big(\wh \theta^{(t)} \big) ,
\eeq
where $\dot{\mL}( \theta )$ and $\ddot{\mL}( \theta ) $ represent the first- and second-order derivatives of the loss function $\mL(\cdot)$ with respect to $ \theta$, respectively, and $\alpha_t$ represents the learning rate. Here, we assume that the initial estimator is close to $\theta_0$, and thus, we set $\alpha_t=1$ \citep{mokhtari2018iqn}. However, for a massive dataset that is distributed on a distributed computing system, efficient execution of the above Newton--Raphson algorithm becomes a non-trivial problem.

One straightforward solution is to retain the original Newton-Raphson algorithm but with distributed computing.
Specifically, we assume that there exist $M$ workers indexed by $1\leq m\leq M$. We denote the entire sample by $\mS_F=\{1,2,...,N\}$ and the sample allocated to the $m$-th worker by $\mS_{(m)}\subset\mS_F$. Then, we have $\cup_{m=1}^M \mS_{(m)}=\mS_F$ and $\mS_{(m_1)}\cap \mS_{(m_2)}=\emptyset$ for any ${m_1}\not ={m_2}$. Given $\wh \theta^{(t)}$, we can then compute the first- and second-order derivatives of the loss function as follows:
\beq
\label{eq:D1D2}
\dot{\mL}\big(\wh \theta^{(t)} \big) = M^{-1} \sum_{m=1}^M \dot{\mL}_{(m)}\big(\wh \theta^{(t)} \big) \quad \text{ and } \quad
\ddot{\mL}\big(\wh \theta^{(t)} \big) = M^{-1} \sum_{m=1}^M \ddot{\mL}_{(m)}\big(\wh \theta^{(t)} \big), \nonumber
\eeq
where $\dot{\mL}_{(m)}\big(\wh \theta^{(t)} \big) = \sum_{i \in \mS_{(m)}} \dot{\ell}( X_i,Y_i; \wh \theta^{(t)} ) $ and $\ddot{\mL}_{(m)}\big(\wh \theta^{(t)} \big) = \sum_{i \in \mS_{(m)}} \ddot{\ell} ( X_i,Y_i; \wh \theta^{(t)} )$. $\ \ $ $\dot{\ell} ( X_i,Y_i; \wh \theta^{(t)} ) $ and $\ddot{\ell} ( X_i,Y_i; \wh \theta^{(t)} ) $ are computed on the $m$-th worker
and are transferred to the central computer for updating $\wh \theta^{(t+1)}$, according to (\ref{eq:Newton-Raphson}). The solution is easy to implement and useful in practical applications but has several serious limitations. First, inverting the $p\times p$-dimensional Hessian matrix using the central computer incurs a computation cost on the order of $O(p^3)$ for each iteration. Second, transferring the local Hessian matrices from each worker to the central computer incurs a communication cost of order $O(p^2)$ for each worker in each iteration. Thus, this approach could incur high
computation and communication costs for high-dimensional data \citep{fan2019distributed}.

Consequently, various communication-efficient Newton-type methods have been proposed to alleviate high communication costs. The underlying key idea is to maximally reduce the number of iterations required to transfer the Hessian matrix. For example, various one-step estimators have been proposed \citep{huang2019distributed,wangfei2020efficient,zhu2021least}. For these methods, only one round of Hessian matrix communication is needed. The resulting estimator can be statistically as efficient as the global one under appropriate regularity conditions.
Methods avoiding Hessian matrix transmission have also been developed \citep{shamir2014communication,zhang2015disco,wang2018giant,crane2019dingo,jordan2019communication}. The underlying key idea is to approximate the entire sample Hessian matrix using an appropriate local estimator, which is computed on a single computer (e.g., the central computer). Consequently, the communication cost resulting from Hessian transmission can be avoided. The inspiration for most statistical research on these methods is to obtain an estimator with statistical efficiency comparable to that of the global one within a small number of iterations. In this manner, the communication cost could be significantly reduced.

Nevertheless, the computation cost for calculating the Hessian inverse matrix is still of order $O(p^3)$.
On one hand, to avoid matrix inverse calculation, distributed gradient descent algorithms have been developed \citep{goyal2017accurate,lin2018distributed,qu2019accelerated,su2019securing}, which require only first-order derivatives of the loss function (i.e., gradients). However, a large number of iterations are typically required for convergence, and the choice of hyperparameters is cumbersome. On the other hand, quasi-Newton methods in a distributed manner have been developed to address this problem \citep{chen2014bfgs,Eisen2017,Lee2018,Soori2020dqn}. The key idea behind quasi-Newton methods is to approximate the Hessian inverse in each iteration without actually inverting the matrix \citep{davidon1991variable,goldfarb1970family}.

Specifically, for quasi-Newton methods, given an approximately inverted Hessian
matrix in the $t$-th iteration $\HNT$, we could obtain $\HNTT$ by solving a linear equation, which is referred to as a secant condition \citep{davidon1991variable,goldfarb1970family}:
\beq
\label{eq:secant-condi}
\HNTT \big\{ \dot{\mL} (\wh \theta^{(t+1)}) - \dot{\mL} (\wh \theta^{(t)}) \big\} = ( \wh \theta^{(t+1)} - \wh \theta^{(t)}).
\eeq Unfortunately, the secant condition cannot uniquely determine $\HNTT$.
Two classical solutions have been proposed to solve this problem. The first is {\it symmetric rank one update} \citep[SR1]{davidon1991variable}. The second solution is referred to as {\it symmetric rank two update} \citep[SR2]{goldfarb1970family}, which is also called Broyden--Fletcher--Goldfarb--Shanno (BFGS) update.
The distributed SR1 \citep{Soori2020dqn} and BFGS \citep{chen2014bfgs,Eisen2017} methods are correspondingly designed. The communication cost of these types of methods could have orders as low as $O(p)$ in each iteration. However, multiple rounds of communication are still required. Moreover, most of the existing studies discuss the numerical convergence of distributed quasi-Newton methods; however, discussions on statistical properties are limited.

To address this, we develop a novel distributed quasi-Newton (DQN) learning method that focuses on the statistical efficiency. With the help of a statistical discussion, we demonstrate that the proposed DQN algorithm requires only a small number of communication iterations to produce an estimator that is statistically as efficient as the global one. As a consequence, the proposed estimator is both communicationally and computationally efficient.
Specifically, estimators and approximated Hessian inverses are first locally computed on each worker computer. Then, a communication mechanism is designed so that each worker passes the local Hessian information to the central computer but only in the form of a $p$-dimensional vector. In each iteration, the communication cost is of order $O(p)$, which is the same as that reported in most existing DQN-related studies \citep{chen2014bfgs,Eisen2017,Lee2018,mokhtari2018iqn,Soori2020dqn}.
However, the proposed DQN method requires only a finite number of iterations with statistical guarantees. To be more precise, under the mild condition, i.e., $Np^{2K} (\log p)^{K+1} / n^{2K+2} \to 0$, only $3K$ rounds of iterations are required, where $K$ is a small finite integer. Consequently, the overall costs attributed to communication and computation are statistically guaranteed. By contrast, a diverging number of iterations is required by methods presented in the existing literature.

The remainder of this paper is organized as follows. In Section 2, we present the DQN methodology and theoretical properties. Numerical studies, including simulation experiments and real data analysis, are presented in Section 3. Section 4 concludes the article with a brief discussion. All technical details are delegated to the Appendixes.

\csection{METHODOLOGY}

\csubsection{Quasi-Newton Algorithm}

We first introduce some notations for model definition. We consider a standard master-and-worker type distributed computation system with one central computer and $M$ worker computers. Let us recall that $\mSm$ is the index set of the sample distributed to the $m$-th worker. For convenience, we assume that $|\mSm| = n$ for every $1 \leq m \leq M$. Then, we have $N = nM$. Moreover, we recall that the global loss function is given by $\mL(\theta) = N^{-1} \sum_{i=1}^N \ell (X_i,Y_i;\theta)$. We define $\thetaglo= \argmin_{\theta} \mL(\theta)$ and $\theta_0 = \argmin_\theta E\big\{ \ell (X_i,Y_i;\theta)\big\}$ as the global estimator and true parameter, respectively. Under appropriate regularity conditions \citep{shao1999mathematical}, we have
$
\sqrt N (\thetaglo - \theta_0) \to_d N(0,\Sigma)
$
for some positive definite matrix $\Sigma \in \mR^{p \times p}$ as $N \to \infty$. For example, $\mL(\theta)$ can be defined as twice the negative log-likelihood function. Accordingly, $\thetaglo$ becomes the maximum likelihood estimator (MLE).
Subsequently, we define the local loss function on the $m$-th worker computer as $\mL_{(m)} (\theta) = n^{-1} \sum_{i \in \mS_m} \ell(X_i,Y_i;\theta)$. Let $\thetam = \argmin_\theta \mL_{(m)} (\theta)$ be the estimator locally obtained on the $m$-th worker computer. Further, $\dot{\ell}(X_i,Y_i;\theta) = \partial \ell(X_i,Y_i;\theta) /\partial \theta \in \mR^p, \ddot{\ell}(X_i,Y_i;\theta) = \partial \ell(X_i,Y_i;\theta) / \partial \theta \theta^\top \in \mR^{p \times p}$, and $\dddot{\ell}(X_i,Y_i;\theta) = \partial \operatorname{vec}\big\{ \ddot{\ell}(X_i,Y_i;\theta) \big\} /\partial \theta \in \mR^{p \times p^2}$ denote the first-, second-, and third-order derivatives of $\theta$, respectively. Finally, for any matrix $B \in \mR^{p \times q}$, $\|B\|_2$ is the maximum singular value of $B$. If $B$ is a symmetric matrix, then $\lambda_{\min}(B)$ and $\lambda_{\max}(B)$ represent its minimal and maximal eigenvalues, respectively.

Before presenting the new method, we briefly introduce quasi-Newton methods.
The well-known quasi-Newton methods were developed to address the problem of computation cost \citep{davidon1991variable,goldfarb1970family}. The key idea is to approximate the Hessian inverse without inverting the Hessian matrix. Let $H^{(t)} \in \mR^{p\times p}$ be the approximately inverted Hessian matrix used in the $t$-th iteration. For the standard Newton--Raphson algorithm, we have $\HNT = \{ \ddot{\mL}( \wh \theta^{(t)}) \}^{-1}$. However, for the quasi-Newton algorithm, this is defined in a different but smart manner. Specifically, note that
$
\dot{\mL } (\wh \theta^{(t+1)}) - \dot{\mL } (\wh \theta^{(t)}) \approx
\ddot{\mL} (\wh \theta^{(t)}) (\wh \theta ^{(t+1)} - \wh \theta^{(t)} )
$ based on Taylor's expansion.
This suggests that given $\HNT$, we could obtain $\HNTT$ by solving the secant condition  (\ref{eq:secant-condi}).
As mentioned previously, the secant condition cannot uniquely determine $\HNTT$.
For SR1 update \citep{davidon1991variable}, given $\HNT$, $H^{(t+1)}$ is updated based on the rank one correction of $\HNT$, i.e., $H^{(t+1)} = H^{(t)} + \alpha u u^\top$ for some undetermined coefficient $\alpha \in \mR$ and $u \in \mR^p$. Accordingly, solving $\alpha$ and $u$ using (\ref{eq:secant-condi}), we obtain
\beq
\label{eq:rankone}
\HNTT = \HNT + \frac{ v^{(t)} \big\{ v^{(t)}\big\}^\top }{ \big\{v^{(t)}\big\}^\top \yt } ,
\eeq
where $ v^{(t)} = \st - \HNT \yt $.
For SR2 update \citep{goldfarb1970family}, given $\HNT$, $H^{(t+1)}$ is updated according to
\beq
\label{eq:BFGS}
\HNTT = \big(\Vt\big)^\top \HNT \Vt + \rhot\Big(\st\Big) \Big(\st \Big)^\top ,
\eeq
where $\rhot = 1/\big[ \big(\st\big)^\top \yt \big]$, $\Vt = I_p - \rhot \yt \big(\st\big) ^\top$, and $I_p\in\mR^p$ is an identity matrix. Equation (\ref{eq:BFGS}) is the well-known BFGS formula \citep{goldfarb1970family}. For convex loss functions, (\ref{eq:BFGS}) guarantees the positive definiteness of $\HNTT$ if $\HNT$ is positive definite \citep{nocedal1999numerical}.
Derivation details for obtaining (\ref{eq:rankone}) and (\ref{eq:BFGS}) are described in Appendix D.

Comparing (\ref{eq:rankone}) and (\ref{eq:BFGS}) with $\{ \ddot{\mL}( \wh \theta^{(t)}) \}^{-1}$, we find that no Hessian matrix inversion is needed for
computing $\HNTT$ using the SR1 or BFGS algorithm if the previous update $\HNT$ is available. Thus, both algorithms offer highly efficient computation. After computing $\HNT$, $\thetat$ can be updated as $\wh \theta^{(t+1)} = \thetat - \HNT \dot{\mL}(\wh \theta^{(t)})$. As proved in \cite{broyden1973local}, the resulting estimator converges Q-superlinearly to the global estimator $\thetaglo$; i.e.,
$
\|\wh \theta^{(t+1)} - \thetaglo \| / \|\thetat - \thetaglo\| \to 0
$ as $t \to \infty$ for a strongly convex loss function. This convergence rate is slightly lower than the quadratic rate of the classical
Newton--Raphson algorithm. However, it is much faster than the linear
rate of various gradient-based methods. This makes the quasi-Newton algorithm one of the most popular algorithms in practice \citep{nocedal1999numerical}.

\csubsection{Distributed One-Stage Quasi-Newton Estimator}

To avoid multiple rounds of iterations, we consider the distributed one-stage quasi-Newton estimator. We were motivated to do so for two reasons. First, as mentioned previously, the quasi-Newton method is computationally efficient, because no Hessian matrix inversion is involved. This makes it particularly attractive for high-dimensional data analysis. Second, with an interesting modification, we find that the quasi-Newton algorithm can operate with a master-and-worker type distributed computing system in a very natural and comfortable manner. The resulting communication cost is also minimal, that is, of order $O(p)$. Specifically, we present a communication and computation efficient distributed one-stage quasi-Newton algorithm, which can be executed using the following three steps.

{\sc Step 1.} At the start, we assume that an initial estimator can be provided for the central computer. The initial estimator should be convenient to obtain. Moreover, it needs to be consistent, but excellent statistical efficiency is not necessary.
Here, we consider the popularly used one-shot estimator \citep{zhang2013communication} as the initial estimator. Accordingly, we need each client to report a local estimator $\thetam$ determined by the quasi-Newton algorithm to the central computer. Next, a global estimator can be simply assembled as $ \thetap = M^{-1} \sum_{m=1}^M \thetam$. Once the initial estimator $\thetap$ is obtained, it is then broadcasted to every worker computer; see the left panel of Figure \ref{fig:1}. This completes the first round of communication with $O(p)$ cost.
\begin{figure}[h]
  \centering
  \includegraphics[width=0.95\textwidth]{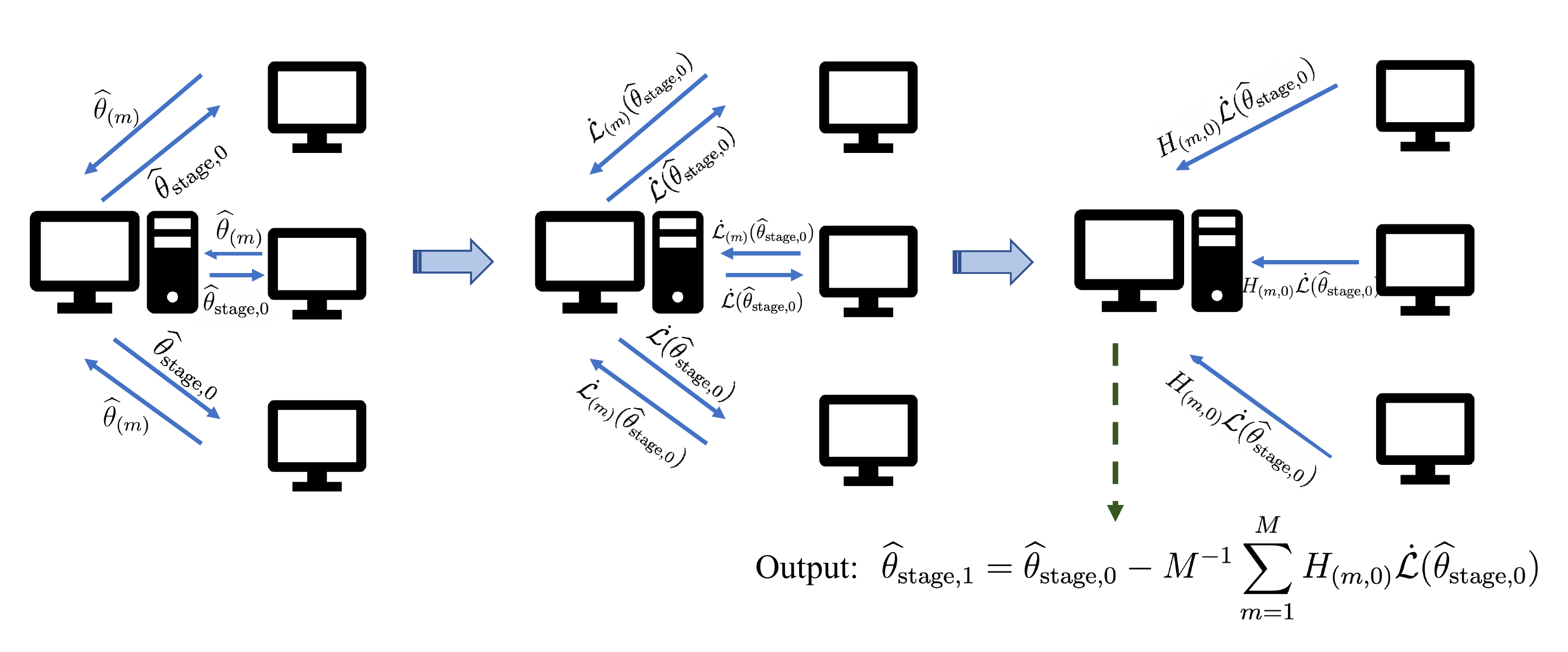}
  \caption{Illustration of the communication-efficient one-stage method.}
  \label{fig:1}
  \end{figure}

{\sc Step 2.} After receiving $\thetap$ from the central computer, each worker computer can compute the local gradients $\dot{\mL}_{(m)} (\thetap)$. These are then transferred back to the central computer to determine the global gradient $\dot{\mL}(\thetap) = M^{-1} \sum_{m=1}^M $ $\dot{\mL}_{(m)} (\thetap)$. Thereafter, the global gradient $\dot{\mL}(\thetap)$ should be broadcasted back to each worker computer. The middle panel of Figure \ref{fig:1} presents an illustration of this second step. This completes the second round of communication with $O(p)$ cost.

{\sc Step 3.} When deriving $\thetam$ in {\sc Step 1,} the approximated Hessian inverse $H_{(m,0)}$ is also obtained as a byproduct of the quasi-Newton algorithm. We apply $H_{(m,0)}$ to the global gradient vector $\dot{\mL}(\thetap)$ to obtain a $p$-dimensional vector $H_{(m,0)} \dot{\mL}(\thetap)$, which is then reported back to the central computer. This completes the third round of communication with $O(p)$ cost. Recall that the central computer also holds the initial estimator $\thetap$. Then, a new estimator can be obtained as follows:
\beq
\label{eq:one-step}
\thetaos = \thetap - M^{-1} \sum_{m=1}^M H_{(m,0)} \dot{\mL}(\thetap). \nonumber
\eeq
The right panel of Figure \ref{fig:1} presents an illustration of the last step. For convenience, we refer to $\thetaos $ as the one-stage DQN estimator (DQN(1)).

\begin{algorithm}
\caption{Distributed One-Stage Quasi-Newton Algorithm}\label{alg:one-step}
\KwIn{Initial estimators $\wh \theta^{(0)}_{(m)}, H^{(0)}_{(m,0)}$ on the $m$-th worker, tolerance $\delta>0$, maximum iterations $T$\;}
\KwOut{One-stage estimator $ \thetaos$\;}
\For{$m=1,2,\dots,M$ (distributedly)}{
\While{ $\operatorname{tol} > \delta$ and $t < T$ }{
$\wh \theta^{(t+1)}_{(m)} = \wh \theta^{(t)}_{(m)} - H^{(t)}_{(m,0)} \dot{\mL}_{(m)}\big(\wh\theta^{(t)}_{(m)}\big)$, where $H^{(t)}_{(m,0)}$ is updated by (\ref{eq:rankone}) or (\ref{eq:BFGS})\;
}
Save $H_{(m,0)}^{(t)}$ and $\wh \theta^{(t)}_{(m)}$ at convergence as $H_{(m,0)}$ and $\thetam$ and then transfer $\thetam$ to the central computer\;
}
The central computer computes $\thetap = M^{-1} \sum_{m=1}^M \wh \theta_{(m)}$ and broadcasts $\thetap$ to each worker\;
\For{$m=1,2,\dots,M$ (distributedly)}{
Compute $\dot{\mL}_{(m)}(\thetap) $ and
transfer it to the central computer\;
}
The central computer computes $\dot{\mL}(\thetap)= M^{-1} \sum_{m=1}^M \dot{\mL}_{(m)}(\thetap)$ and broadcasts $\dot{\mL}(\thetap)$ to each worker\;
\For{$m=1,2,\dots,M$ (distributedly)}{
Calculate $H_{(m,0)} \dot{\mL} (\thetap)$ and
transfer it to the central computer\;
}
The central computer computes
$
\thetaos = \thetap - M^{-1} \sum_{m=1}^M H_{(m,0)} \dot{\mL} (\thetap).
$

\end{algorithm}

To summarize, three rounds of master-and-worker communication are needed to compute $\thetaos$. Because the communication cost for each round is of order $O(p)$, the total communication cost is also of the same order, which is the lowest communication complexity possible for a $p$-dimensional distributed parameter estimation problem. A more detailed description of the algorithm is given in Algorithm \ref{alg:one-step}. Note that $\thetaos$ shares a similar spirit as the classical one-step estimator for MLE \citep{van2000asymptotic}. However, $\thetaos$ is mainly designed for a distributed system with minimal communication and computation costs.

\csubsection{Theoretical Properties}

We next study the theoretical properties of the proposed DQN(1) estimator.
To this end, several regularity conditions must be considered.

\begin{itemize}
\item [(C1) ] ({\sc Randomness}) Assume that ($X_i,Y_i$)s on the $m$-th worker are independently and identically distributed.

\item [(C2)] ({\sc Parameters})
The parameter space $\mathbf{\Theta} $ is a compact and convex subset of $\mR^p$. In addition, $\theta_0 \in \operatorname{int} ( \mathbf{\Theta})$ and $R:= \sup_{\theta \in \mathbf{\Theta}} \| \theta- \theta_0 \| > 0$.

\item [(C3)] ({\sc Local Strong Convexity})
Define $\Omega(\theta) = E\Big[ \dot{\ell}(X_i,Y_i;\theta) \big\{\dot{\ell}(X_i,Y_i;\theta) \big\}^\top \Big] =-  E\big\{ \ddot{\ell}(X_i,Y_i;\theta) \big\}$. Assume $ \tau_{\min} \leq \lambda_{\min}\big\{\Omega(\theta_0) \big\} \leq \lambda_{\max}\big\{\Omega(\theta_0) \big\} \leq\tau_{\max} $ for some positive constants $\tau_{\min}$ and $\tau_{\max}$.

\item [(C4)] ({\sc Smoothness}) Define $B(\theta_0, \delta) = \big\{ \theta \in \mR^p \big| \|\theta - \theta_0 \| \leq \delta \big\}$ to be a ball around $\theta_0$ with radius $\delta > 0$. Assume that
there exist two constants $C_G>0$ and $C_H>0$ such that the following inequalities hold.
$$
E\bigg\{ \Big\| \dot{\ell}(X_i,Y_i;\theta) \Big\|_2^8 \bigg\} \! \leq \! C_G^8, E\bigg\{\Big\| \ddot{\ell}(X_i,Y_i;\theta) - \Omega(\theta) \Big\|_2^8 \bigg\} \! \leq \! C_H^8 \text{ for all } \theta \!\in\! B(\theta_0,\delta).
$$
Moreover, for all $X \in \mR^p, Y\in \mR$, $\ddot{\ell}(X_i,Y_i;\theta) $ and $\dddot{\ell}(X_i,Y_i;\theta) $ are both Lipschitz continuous, in the sense that for any $\theta^\prime, \theta^{\prime \prime} \in B(\theta_0,\delta)$ and $u \in \mR^p$,
\begin{gather}
\Big\|
\ddot{\ell}(X_i,Y_i;\theta^{\prime}) - \ddot{\ell}(X_i,Y_i;\theta^{\prime\prime} ) \Big\|_2 \leq C(X_i,Y_i) \big\| \theta^{\prime} - \theta^{\prime\prime} \big\| \text{ and } \nonumber\\
\Big\|
\big\{ \dddot{\ell}(X_i,Y_i;\theta^{\prime}) - \dddot{\ell}(X_i,Y_i;\theta^{\prime\prime} )\big\} ( u \otimes I_p) \Big\|_2 \leq C(X_i,Y_i) \big\| \theta^{\prime} - \theta^{\prime\prime} \big\| \|u\|, \nonumber
\end{gather}
and $E\big\{ C^8(X_i,Y_i) \big\} \leq C_{\max}^8 $, $E\big[ C^8(X_i,Y_i) - E\big\{ C(X_i,Y_i) \big\}^8  \big] \leq C_{\max}^8$ for some positive constant $C_{\max}$.
\item [(C5)]  ({\sc Convergence})
For the $m$-th worker, define the $t$-th step local approximate Hessian inverse to be $H_{(m,0)}^{(t)}$, assume that
$ \lim_{t \to \infty} \big\|H_{(m,0)}^{(t)} - \big\{\ddot{\mL}(\thetam)\big\}^{-1} \big\|_2 \to 0$.
\item [(C6)]  ({\sc Dimensionality}) We assume that $p \sqrt{\log p} /n \to 0$ as $n \to \infty$.
\end{itemize}
Condition (C1) requires that the data be randomly distributed across different computers to ensure the statistical consistency of the one-shot estimator as a convenient initial estimator. The same condition was adopted in \cite{zhang2013communication} and \cite{fan2019distributed}. (C2)--(C4) are classical regularity conditions in convex optimization \citep{zhang2013communication,jordan2019communication}. (C5) guarantees the convergence of the approximation matrix $H_{(m,0)}$ and has previously been rigorously investigated. Specifically, for SR1 update, (C5) has been rigorously proved by \cite{conn1991convergence}, assuming that sequence $\big\{H_{(m,0)}^{(t)} \dot{\mL}(\wh \theta^{(t)})\big\}$ is uniformly linearly independent. (C5) has also been verified for
BFGS update by  \cite{schuller1974order}  under slightly stronger conditions.
(C6) specifies the relationship between the dimension $p$ and local data size $n$.
Given these technical conditions, we establish Theorem \ref{thm:one-step}.

\bet
\label{thm:one-step}
Assume that (C1)--(C6) hold. Then, we have $\big\| \thetaos - \thetaglo \big\| \leq \kappa  \Big(  M^{-1} $   $ \sum_{m=1}^M \Big[ \|\thetam - \theta_0\|^2 +  \| \hesmtrue - \Omega(\theta_0) \|_2^2 + \big\| \big\{ \dddot{\mL}_{(m)}(\theta_0) - \dot{\Omega}(\theta_0) \big\} $ $\big\{ (\thetam - \theta_0) \otimes I_p \big\} \big\|_2 \Big] + \| \thetap - \theta_0 \|  \Big)  \big \| \thetap - \theta_0 \big\|$ for some constant $\kappa>0$ with probability tending to one. Further, assuming that $N (p\log p)^2/n^4 \to 0$, we have $\|\thetaos - \thetaglo \| = o_p(N^{-1/2})$.
\eet
\noindent
The detailed proof is given in Appendix A.1. From Theorem \ref{thm:one-step}, we infer that the discrepancy between $\thetaos$ and $\thetaglo$ is upper bounded by $\Big(  M^{-1} \sum_{m=1}^M \Big[ \|\thetam - \theta_0\|^2 +  \| \hesmtrue - \Omega(\theta_0) \|_2^2 + \big\| \big\{ \dddot{\mL}_{(m)}(\theta_0) - \dot{\Omega}(\theta_0) \big\}   \big\{ (\thetam - \theta_0) \otimes I_p \big\} \big\|_2 \Big] + \| \thetap - \theta_0 \|  \Big) \big \| \thetap - \theta_0 \big\|$, =$O_p( p \log p /n^2) + o_p(1/\sqrt{N})$. Thus, the difference $\|\thetaos-\thetaglo\|$ is further reduced compared with $\|\thetap - \theta_0 \|$ to order $O_p(1/\sqrt{N} + \sqrt{\log p}/n)$; see details in equation (\ref{apa:one-stage-order}) of Appendix A. The amount of compression is determined by three factors: (1) averaged distance of the local estimator $M^{-1} \sum_{m=1}^M \|\thetam - \theta_0\|$; (2) averaged distance of the local estimator, Hessian matrix, and third derivative matrix $M^{-1} \sum_{m=1}^M \Big[ \|\hesmtheta - \Omega(\theta_0)\|_2 + \big\| \big\{ \dddot{\mL}_{(m)}(\theta_0) - \dot{\Omega}(\theta_0) \big\} \big\{ (\thetam - \theta_0) \otimes I_p \big\}\big\|_2\Big]$; and (3) distance between the initial estimator and true parameter $\|\thetap - \theta_0\|$.
Accordingly, assuming $N (p\log p)^2/n^4 \to 0$,
$\thetaos$ achieves the optimal statistical efficiency. When $p$ is fixed, this condition reduces to $N /n^4 \to 0$. It is a condition much weaker than $N/n^2 \to 0$, which has been typically assumed in the existing literature \citep{zhang2013communication,wangfei2020efficient}.


\csubsection{Distributed Multi-Stage Quasi-Newton Estimator}

In the previous section, we introduced the DQN(1) estimator.
Note that to achieve the optimal statistical efficiency, we require $N(\log p)^4/n^4 \to 0$. This condition can be easily satisfied if the feature dimension $p$ is not too high. By contrast, if $p$ is relatively high, the convergence rate of the DQN(1) estimator slows down. To fix this, we further develop a multi-stage DQN estimator. First, we present a two-stage DQN estimator with two extra updating steps with BFGS update and refer to it as the DQN(2) estimator. The details of the DQN(2) algorithm are given below. It is remarkable that, after the first three steps, the DQN(1) estimator $\thetaos$ is already computed by the central computer.

{\sc Step 4.} Broadcasting the DQN(1) estimator to each worker computer. Similar to the DQN(1) algorithm, the worker computer should compute the local gradient $\dot{\mL}_{(m)}(\thetaos)$, which should be reported back to the central computer. As a consequence, the global gradient $\dot{\mL}(\thetaos)$ can be assembled. This leads to two rounds of communication with a cost of order $O(p)$.

{\sc Step 5.} Note that, when we compute the DQN(1) algorithm, each worker holds an approximated Hessian inverse matrix $H_{(m,0)}$. Moreover, note that $\thetap$ and $\dot{\mL}(\thetap)$ are the estimators obtained in the process of the DQN(1) algorithm for each worker. Consequently, given $H_{(m,0)}$, each worker could compute the updated matrix $H_{(m,1)}$ according to the BFGS formula (\ref{eq:BFGS}) as follows:
\beq
\label{eq:two-step-update}
H_{(m,1)} = \big\{ V_0 \big\}^\top H_{(m,0)} \big\{ V_0 \big\} + \rho_0 \Big( \thetaos - \thetap \Big) \Big( \thetaos - \thetap \Big)^\top,
\eeq
where $V_0 = I_p - \rho_0 \yzero \big(\szero\big) ^\top$ and $\rho_0 = 1/\big[\big(\szero\big) ^\top \yzero \big]$. After computing $H_{(m,1)} $, it is applied to the global gradient $\dot{\mL}(\thetaos)$. This leads to a $p$-dimensional vector $H_{(m,1)} \dot{\mL}(\thetaos)$, which is then reported back to the central computer. Subsequently, the DQN(2) estimator could be derived as
\beq
\label{eq:bfgs-two-step}
\thetatwo = \thetaos - M^{-1} \sum_{m=1}^M H_{(m,1)} \dot{\mL} \Big(\thetaos\Big).
\eeq
Thus, Steps 4 and 5 constitute the second stage estimation. The detailed algorithm is given in Algorithm \ref{alg:two-step}. Moreover, a two-stage DQN estimator with the SR1 updating strategy could be similarly obtained; more details are presented in Appendix C.1.

\begin{algorithm}
\caption{Distributed Two-Stage Quasi-Newton Algorithm}\label{alg:two-step}
\KwIn{DQN(1) estimator $\thetaos$ on the central computer, $\thetap$, $ \dot{\mL} (\thetap)$, and the initial Hessian inverse approximation $H_{(m,0)}$ on the $m$-th worker\; }
\KwOut{DQN(2) estimator $\thetatwo$\;}
The central computer broadcasts $\thetaos$ to each worker\;
\For{$m=1,2,\dots,M$ (distributedly)}{
Compute $\dot{\mL}_{(m)}(\thetaos) $ and
transfer it to the central computer\;
}
The central computer computes $\dot{\mL}(\thetaos)= M^{-1} \sum_{m=1}^M \dot{\mL}_{(m)}(\thetaos)$ and broadcasts $\dot{\mL}(\thetaos)$ to each worker\;
\For{$m=1,2,\dots,M$ (distributedly)}{
Update $H_{(m,1)}$ according to (\ref{eq:two-step-update})\;
Calculate $H_{(m,1)} \dot{\mL} (\thetaos)$ and
transfer it to the central computer\;
}
The central computer computes
$
\thetatwo = \thetaos - M^{-1} \sum_{m=1}^M H_{(m,1)} \dot{\mL} (\thetaos).
$
\end{algorithm}

Similar to Algorithm \ref{alg:one-step}, Algorithm \ref{alg:two-step} incurs another three rounds of communication. Recall that three extra rounds of communication are needed for computing $\thetaos$. Thus, a total of six rounds of communication are needed for computing $\thetatwo$. The communication cost for each round remains of order $O(p)$. Consequently, the total communication cost of the two-stage estimator remains of order $O(p)$. Additionally, no Hessian matrix needs to be inverted. However, a better estimation accuracy could be achieved due to the additional updating stage, which leads to the next theorem.
\bet
\label{thm:one-step-updating}
Assume that the technical conditions (C1)--(C6) hold. Then, we have $\big\| \thetatwo - \thetaglo \big\| \leq \kappa_2 \Big(M^{-1} \sum_{m=1}^M \Big[ \|\thetam - \theta_0\|^2 + \| \hesmtrue - \Omega(\theta_0) \|_2^2 + \big\| \big\{ \dddot{\mL}_{(m)}(\theta_0) - \dot{\Omega}(\theta_0) \big\} \big\{ (\thetam - \theta_0) \otimes I_p \big\} \big\|_2 \Big] + \| \thetap - \theta_0 \| \Big) \big \| \thetaos - \thetaglo \big\|$ for some constant $\kappa_2>0$ with probability tending to one. Further, assuming that $Np^4 (\log p)^3/n^6 \to 0$, we have $\|\thetatwo - \thetaglo\| = o_p(N^{-1/2})$.
\eet
\noindent
The proof of Theorem 2 is given in Appendix A.2.
It could be verified that the discrepancy between $\thetatwo$ and $\thetaglo$ is further reduced from
$\|\thetaos - \thetaglo\| = O_p( p \log p /n^2) + o_p(1/\sqrt{N})$ to $ \|\thetatwo - \thetaglo\| = O_p(p^2(\log p)^{3/2}/n^3) + o_p(1/\sqrt{N})$ ; see Appendix A for more details. Accordingly, the optimal statistical efficiency can be achieved if
$Np^4 (\log p)^3/n^6 \to 0$. This is a weaker condition than that of the DQN(1) estimator. Next, we extend the idea of the DQN(2) estimator to develop the multi-stage DQN (DQN($K$)) estimator $\thetaK$. The detailed algorithm is given in Algorithm \ref{alg:multi-step}. The theoretical properties are summarized by Corollary \ref{thm:K-step-updating}.

\begin{algorithm}
\caption{Distributed K-Stage Quasi-Newton Algorithm}\label{alg:multi-step}
\KwIn{DQN($K\!-\!1$) estimator $\thetakone$ on the central computer, $\thetaktwo$, $ \dot{\mL} (\thetaktwo)$, and Hessian inverse approximation $H_{(m,K-2)}$ on the $m$-th worker\; }
\KwOut{ DQN($K$) estimator $\thetaK$\;}
The central computer broadcasts $\thetakone$ to each worker\;
\For{$m=1,2,\dots,M$ (distributedly)}{
Compute $\dot{\mL}_{(m)}(\thetakone) $ and
transfer it to the central computer\;
}
The central computer computes $\dot{\mL}(\thetakone)= M^{-1} \sum_{m=1}^M \dot{\mL}_{(m)}(\thetakone)$ and broadcasts $\dot{\mL}(\thetakone)$ to each worker\;
\For{$m=1,2,\dots,M$ (distributedly)}{
Compute $H_{(m,K-1)} = \big\{ V_{K-2} \big\}^\top H_{(m,K-2)} \big\{ V_{K-2} \big\} + \rho_{K-2} \big( \thetakone - \thetaktwo \big) \big( \thetakone - \thetaktwo \big)^\top$\;
Calculate $H_{(m,K-1)} \dot{\mL} (\thetaos)$ and
transfer it to the central computer\;
}
The central computer obtains
$
\thetaK = \thetakone - M^{-1} \sum_{m=1}^M H_{(m,K-1)} \dot{\mL} (\thetakone).
$
\end{algorithm}

\begin{corollary}
\label{thm:K-step-updating}
Assume that the technical conditions (C1)--(C6) hold. Then, we have
$\big\| \thetaK - \thetaglo \big\| \leq \kappa_{_{K}}  \Big(  M^{-1} \sum_{m=1}^M \Big[ \|\thetam - \theta_0\|^2 +  \| \hesmtrue - \Omega(\theta_0) \|_2^2 + \big\| \big\{ \dddot{\mL}_{(m)}(\theta_0) - \dot{\Omega}(\theta_0) \big\} \big\{ (\thetam - \theta_0) \otimes I_p \big\} \big\|_2 \Big] + \| \thetap - \theta_0 \|  \Big)^K \big \| \thetap - \theta_0 \big\|$ for some constant $\kappa_{_{K}}>0$ with probability tending to one. Further, assuming that $Np^{2K} (\log p)^{K+1} / n^{2K+2} \to 0$, we have $\|\thetaK - \thetaglo \| = o_p(N^{-1/2})$.
\end{corollary}
\noindent
The proof of Corollary \ref{thm:K-step-updating} is given in Appendix A.3. It could be found that Corollary \ref{thm:K-step-updating} is a directly generalized version of Theorem \ref{thm:one-step-updating}. To be more specific, the discrepancy between the DQN(K) estimator $\thetaK$ and $\thetaglo$ is further compressed from $\|\thetap - \thetaglo\|$ by $\Big(  M^{-1} \sum_{m=1}^M \Big[ \|\thetam - \theta_0\|^2 +  \| \hesmtrue - \Omega(\theta_0) \|_2^2 + \big\| \big\{ \dddot{\mL}_{(m)}(\theta_0) - \dot{\Omega}(\theta_0) \big\} \big\{ (\thetam - \theta_0) \otimes I_p \big\} \big\|_2 \Big] + \| \thetap - \theta_0 \|  \Big)^K$.
Consequently, the optimal statistical efficiency can be obtained using the DQN($K$) estimator with even weaker technical conditions. In other words, $Np^{2K} (\log p)^{K+1} / n^{2K+2} \to 0$.
Moreover, Algorithm \ref{alg:multi-step} shows that the DQN($K$) estimator requires $3K$ rounds of communication, with cost $O(p)$ for each round. Therefore, practical applications should consider the trade-off between statistical efficiency and time cost.

\csection{NUMERICAL STUDIES}


\csubsection{Performance of the DQN Algorithm}

We start with demonstrating the finite sample performance of the proposed DQN method. Specifically, we present two simulation examples as follows.

\begin{itemize}
\item {\sc Example 1. (Logistic Regression)} We consider a logistic regression, which is one of the most popular classification models. We set $\theta_0 = c_0 \gamma / \| \gamma\|$, where $\gamma \in \mR^p$ is generated from a standard normal distribution, and $c_0 = 1.5$ controls the signal strength. The covariate $X_i$ is generated from a multivariate normal distribution with $E(X_i) = 0$ and $\cov(X_{ij_1}, X_{ij_2}) = \rho^{|j_1 - j_2|}$ with $\rho = 0.5$ for $1 \leq j_1,j_2 \leq p$. Given $X_i$, the response $Y_i \in \{0,1\}$ is then generated according to
$P(Y_i = 1 | X_i,\theta_0) = \{1 + \exp(- X_i^\top \theta_0)\}^{-1}$.
\item {\sc Example 2. (Poisson Regression)} This is an example revised from \cite{fan2001variable}. Specifically, $\theta_0$ and $X_i$ are the same as those in {\sc Example 1} but with $c_0 = 0.3$ and $\rho = 0.2$. Conditional on $X_i$, response $Y_i$ is generated from a Poisson distribution with $ E(Y_i|X_i) = \exp(X_i^\top \theta_0)$.

\end{itemize}
\noindent
For each simulation example, the sample size is $N = 10^6$, and we vary the local data size $n$ and dimension $p$. The generated sample data are randomly distributed to different workers $M = N/n$.
We replicate the experiment $R= 100$ times for reliability.
\begin{table}
\centering
\caption{Log(MSE) values and the corresponding SD, IQR and range values for {\sc Example 1}. The numerical performance is evaluated for different methods with different feature dimensions $p (\times 10^2)$.  The whole sample size $N$ and local sample size $n$ are fixed to be $10^6$ and $2 \times 10^4$, respectively. The reported results are averaged for $R=100$ simulation replications. }\label{tab:mse-p}
\begin{spacing}{1}
\setlength{\tabcolsep}{1.1mm}{
\begin{tabular}{cc|cc|cc|cc|cc|cc|c}
\hline
\hline
& & \multicolumn{2}{c|}{Stage 0} & \multicolumn{2}{c|}{Stage 1} & \multicolumn{2}{c|}{Stage 2} & \multicolumn{2}{c|}{Stage 3} & \multicolumn{2}{c|}{Stage 4} &\\
&$p$ &SR1 &BFGS &SR1 &BFGS&SR1 &BFGS&SR1 &BFGS&SR1 &BFGS&MLE \\[0.1em]
\hline
log&1&-6.86&-6.86&-6.96&-6.96&-6.96&-6.96&-6.96&-6.96&-6.96&-6.96&-6.96\\
(MSE)&10&-3.91&-3.91&-4.61&-4.61&-4.62&-4.62&-4.64&-4.64&-4.65&-4.65&-4.65\\
&20&-2.65&-2.65&-3.87&-3.87&-3.90&-3.90&-3.93&-3.93&-3.96&-3.96&-3.96\\
\multicolumn{13}{c}{} \\
SD&1&0.15&0.15&0.16&0.16&0.16&0.16&0.16&0.16&0.16&0.16&0.16\\
&10&0.05&0.05&0.05&0.05&0.05&0.05&0.05&0.05&0.05&0.05&0.05\\
&20&0.03&0.03&0.04&0.04&0.04&0.04&0.04&0.04&0.04&0.04&0.04\\
\multicolumn{13}{c}{} \\
IQR&1&0.19&0.19&0.19&0.19&0.20&0.20&0.21&0.21&0.20&0.20&0.19\\
&10&0.06&0.06&0.05&0.05&0.05&0.05&0.05&0.05&0.06&0.06&0.06\\
&20&0.04&0.04&0.05&0.05&0.05&0.05&0.05&0.05&0.05&0.05&0.05\\
\multicolumn{13}{c}{} \\
range&1&0.68&0.68&0.69&0.69&0.69&0.69&0.70&0.70&0.70&0.70&0.70\\
&10&0.26&0.26&0.23&0.23&0.23&0.23&0.32&0.32&0.23&0.23&0.23\\
&20&0.18&0.18&0.20&0.20&0.20&0.20&0.19&0.19&0.19&0.19&0.20\\
\hline
\end{tabular}}
\end{spacing}
\end{table}

 To gauge the finite sample performance of the proposed method, various performance analyses are developed. Specifically,
let $\wh \theta_{\operatorname{stage,K}}^{(r)}$ be the DQN($K$) estimator (by SR1 or BFGS updating) obtained in the $r$-th replication. Then, the mean squared error (MSE) is defined as $\operatorname{MSE}= R^{-1} \sum_{r=1}^R \| \wh \theta_{\operatorname{stage,K}}^{(r)} - \theta_0\|^2$. Moreover, a total of four measures are developed to evaluate the estimator's stability and robustness. They are the MSE values in log-scale (i.e., $\log$(MSE)), standard deviation (SD) of $\log$(MSE), inter-quartile range (IQR) of $\log$(MSE), and range of $\log$(MSE). The detailed results are given in Tables \ref{tab:mse-p} and \ref{tab:mse-n}.
Because simulation results of {\sc Example 1} are quantitatively similar to those of {\sc Example 2}, we report the results for {\sc Example 1} only. The detailed results for {\sc Example 2} are given in Appendix E.

From Table \ref{tab:mse-p}, we find that the values of all four measures increase as $p$ decreases for a fixed $n$. By contrast, from Table \ref{tab:mse-n}, we find that, with a fixed $p$, larger $n$ always leads to an improved estimation performance in the sense that all four measure values approach those of MLE (i.e., $\thetaglo$). Moreover, when $p$ is relatively small or $n$ is relatively large, the $\log$(MSE) value of $\wh \theta_{\operatorname{stage},1}$ or $\wh \theta_{\operatorname{stage},2}$ is
comparable with that of MLE. However, as $p$ grows (or $n$ drops), more stages (i.e., larger $K$) are required to obtain an estimator with optimal statistical efficiency. Nevertheless, the number of required stages $K$ remains very small (e.g., $K \leq 4$). Thus, the algorithm is communicationally and computationally efficient. The SD, IQR, and range values also demonstrate similar patterns. These results are consistent with our theoretical findings in Theorems 1 and 2 and Corollary 1.

\begin{table}[h]
\caption{Log(MSE) values and the corresponding SD, IQR and range values for {\sc Example 1}. The numerical performance is evaluated for different $n (\times 10^2)$ and methods.  The whole sample size $N$ and feature dimension $p$ are fixed to $N = 10^6,$ and $p = 10 ^3$, respectively. Finally, the reported results are averaged based on $R=100$ simulations.}\label{tab:mse-n}
\centering
\begin{spacing}{1}
\setlength{\tabcolsep}{1.1mm}{
\begin{tabular}{cc|cc|cc|cc|cc|cc|c}
\hline
\hline
& & \multicolumn{2}{c|}{Stage 0} & \multicolumn{2}{c|}{Stage 1} & \multicolumn{2}{c|}{Stage 2} & \multicolumn{2}{c|}{Stage 3} & \multicolumn{2}{c|}{Stage 4} &\\
&$n$ &SR1 &BFGS &SR1 &BFGS&SR1 &BFGS&SR1 &BFGS&SR1 &BFGS&MLE \\[0.1em]
\hline
log&50&-2.97&-2.97&-5.23&-5.23&-5.12&-5.12&-5.29&-5.29&-5.32&-5.32&-5.34\\
(MSE)&100&-4.23&-4.23&-5.28&-5.28&-5.30&-5.30&-5.32&-5.32&-5.33&-5.33&-5.34\\
&500&-5.25&-5.25&-5.33&-5.33&-5.33&-5.33&-5.33&-5.33&-5.33&-5.33&-5.34\\
\multicolumn{13}{c}{} \\
SD&50&0.04&0.04&0.08&0.08&0.29&0.29&0.08&0.08&0.08&0.08&0.08\\
&100&0.05&0.05&0.08&0.08&0.08&0.08&0.08&0.08&0.08&0.08&0.08\\
&500&0.07&0.07&0.08&0.08&0.08&0.08&0.08&0.08&0.08&0.08&0.08\\
\multicolumn{13}{c}{} \\
IQR&50&0.04&0.04&0.10&0.10&0.18&0.18&0.10&0.10&0.10&0.10&0.09\\
&100&0.07&0.07&0.10&0.10&0.10&0.10&0.09&0.09&0.09&0.09&0.09\\
&500&0.09&0.09&0.09&0.09&0.09&0.09&0.09&0.09&0.09&0.09&0.09\\
\multicolumn{13}{c}{} \\
range&50&0.22&0.22&0.40&0.40&1.33&1.33&0.41&0.41&0.42&0.42&0.43\\
&100&0.31&0.31&0.40&0.40&0.40&0.40&0.41&0.41&0.42&0.42&0.43\\
&500&0.42&0.42&0.43&0.43&0.43&0.43&0.43&0.43&0.43&0.43&0.43\\
\hline
\end{tabular}}
\end{spacing}
\end{table}
\newpage

\csubsection{Comparison with Competing Methods}

We next compare the proposed method with the following four competing methods: (1) the distributed one-step Newton (DOSN) estimator of \cite{huang2019distributed},
(2) the communication-efficient surrogate likelihood (CSL) based estimator of \cite{jordan2019communication}, (3) the distributed momentum gradient descent (DMGD) estimator of \cite{goyal2017accurate}, and (4) the distributed asynchronous averaged quasi-Newton (DAQN) estimator of \cite{Soori2020dqn}. The simulation model used here is the same as that in Section 3.1. We fix the sample size to be $N = 10^6$, the number of workers to be $M=50$, and vary the dimension $p$ from $500$ to $2500$. Moreover, we set $K=4$ for {\sc Example 1} and $K=2$ for {\sc Example 2}. We replicate the experiment a total of $R= 100$ times.

 To gauge the finite performance of different methods, we consider four different performance measures. First, to measure the {\it estimation accuracy}, we focus on the $\log$(MSE) values. Second, to compare {\it computation efficiency}, we record for each method the computing time for the master plus the averaged computing time for each worker as $T_{\rm{1}}^{(r)}$ in the $r$-th ($1\leq r\leq R$) simulation. Then, the averaged computing time $T_{\rm{1}}$ for the $R$ simulations is calculated and reported. Third, to measure the {\it communication efficiency}, the communication time for each simulation $T_{\rm{2}}^{(r)}$ is estimated by the overall time cost $T^{(r)}$ minus the computing time $ T_{\rm{1}}^{(r)}$. Similarly, the averaged communication time $T_{\rm{2}}$ is calculated and reported. Finally, the averaged {\it total time cost $T$} is also reported for better comparison. The simulation results are reported in Table \ref{tab:tc} and Figure \ref{fig:cpt_methods}.

As shown in Figure \ref{fig:cpt_methods}, all methods demonstrate similar performance in terms of estimate accuracy with similar $\log$(MSE) values. From Table \ref{tab:tc}, the following conclusions could be drawn. {First}, for the {\it computation cost}, we find that (1) $T_1$ for the DMGD method is much larger than that for the other methods, because it requires a considerable number of iterations to converge; (2) the $T_1$ value of CSL increases dramatically as $p$ increases due to Hessian inverse calculation with complexity of $O(p^3)$; and (3) the DQN methods perform well with the lowest $T_1$ values, which is especially true for large $p$.
{Second}, for the {\it communication cost}, we find that the $T_2$ value for the DOSN method is the highest. This is as expected because DOSN needs to transfer a Hessian matrix for calculation. This leads to a complexity of order $O(p^2)$. In contrast, the DQN has the lowest $T_2$ value. Finally, in comparison of the {\it total time cost}, the DQN methods perform the best in terms of $T$. To summarize, the DQN methods demonstrate comparable estimation accuracy and the lowest total time cost.

\begin{table}
\caption{{ Averaged computation cost $T_1$, communication cost $T_2$ and total time cost $T$ for {\sc Examples 1} and 2. The time cost is evaluated for different methods with different feature dimensions $p$.  The whole sample size $N$ and local sample size $n$ are fixed to be $10^6$ and $2 \times 10^4$, respectively. The reported results are averaged for $R=100$ simulation replications. }}\label{tab:tc}
\renewcommand\arraystretch{1.5}
\centering
\begin{spacing}{0.8}
\setlength{\tabcolsep}{1.4mm}{
\begin{tabular}{rr|rrrrcc}
\hline
\hline
&$p$ &DOSN &CSL &DMGD&DAQN &DQN-BFGS&DQN-SR1 \\
\hline
&&\multicolumn{6}{c}{\sc Example 1} \\
$T_1$ & 500&0.95  &1.86&40.29&3.94&0.59&0.68\\
  & 1000&2.71  &5.54&70.49&13.44&1.11&1.36\\
  & 2000&8.69  &24.75&113.59&33.92&3.63&3.72\\
  & 2500&13.14 &45.37&267.96&68.14&5.57&7.27\\
  \hline
$T_2$ & 500&42.16 &0.95&21.36&1.56&0.85&0.67\\
  & 1000&212.94 &1.90&42.82&3.01&1.66&1.26\\
  & 2000&799.87 &3.86&72.77&5.56&3.46&2.42\\
  & 2500&1240.81 & 4.64&90.97&8.36&4.61&3.16\\
  \hline
$T$ & 500&43.12 &2.82&61.65&5.51&1.44&1.35\\
  & 1000&215.64 &7.44&113.32&16.45&2.76&2.62\\
  & 2000&808.56 &28.61&186.35&39.48&7.09&6.15\\
  & 2500&1253.95&  50.01&358.93&76.51&10.17&10.42\\
  \hline
&&\multicolumn{6}{c}{\sc Example 2} \\
$T_1$ & 500&0.50 &0.72&52.09&2.44&0.31&0.31\\
  & 1000&1.67 &2.24&70.20&7.56&0.56&0.60\\
  & 2000&8.26 &10.26&113.54&26.62&1.50&1.53\\
  & 2500&13.22&18.94&144.30&53.25&1.95&2.27\\
  \hline
$T_2$ & 500&40.47  &1.07&19.09&1.00&0.36&0.31\\
  & 1000&166.58 &1.68&41.57&2.20&0.69&0.56\\
  & 2000&775.31&3.09&72.49&5.35&1.42&1.18\\
  & 2500&1308.45&3.72&91.92&6.79&1.80&1.58\\
  \hline
$T$ & 500&40.96&1.79&71.18&3.44&0.67&0.62\\
  & 1000&168.24 &3.92&111.77&9.76&1.25&1.16\\
  & 2000&783.57&13.35&186.03&31.96&2.92&2.71\\
  & 2500&1321.68&22.66&236.23&60.04&3.75&3.85\\
\hline
\end{tabular}}
\end{spacing}
\end{table}
\begin{figure}[h]
\centering
\subfigure{\includegraphics[width=6in]{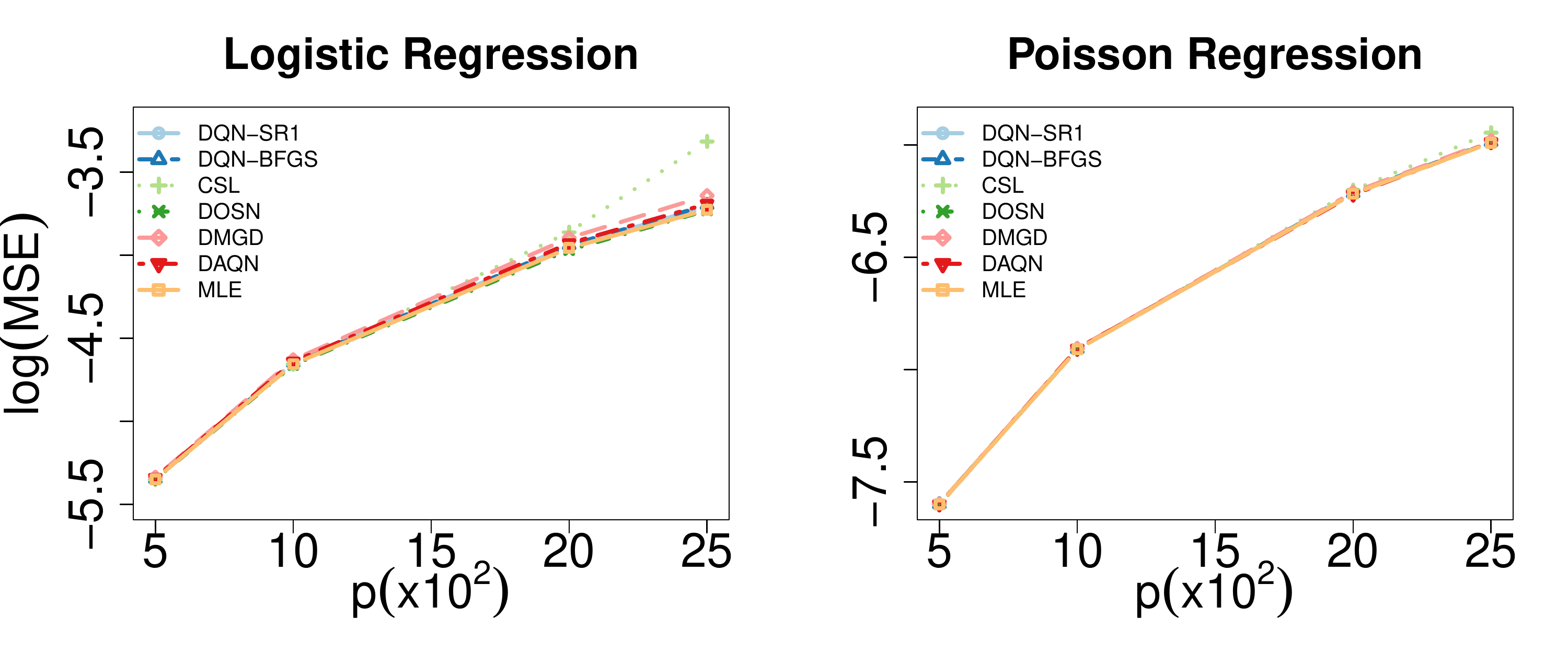} }
\caption{Log(MSE) values for different methods with different dimension $p$. The results for the logistic and Poisson regression models are given in the left and right panels, respectively. The whole sample size $N$ and local subsample size $n$ are fixed at $N = 10^6$ and $n = 2 \times 10^4$. The number of workers $M=50$. The reported log(MSE) values are averages for $R=100$ simulations.}
\label{fig:cpt_methods}
\end{figure}

\csubsection{Ultrahigh-Dimensional Features}

 We next consider the ultrahigh-dimensional feature situation with $p\gg n$. In this case, appropriate sparse structure has to be assumed to the true regression coefficient \citep{Fan:Lv:2008}. Therefore, various screening methods \citep{Fan:Lv:2008,Fan:Song:2010,Li:Peng:Zhang:2012,He:Wang:Hong:2013,Li:Li:2020} can be readily applied but in a distributed way. Once the feature dimension is significantly reduced, the DQN algorithm can be readily applied. For the purpose of illustration, we consider here the sure independence screening method for generalized linear
models \citep{Fan:Song:2010} and calculate the statistic in a distributed way as follows.

We start with a simulation setup as suggested by \cite{Fan:Song:2010}. More specifically, the covariates are generated by
 $X_{ij}=(\epi_{ij}+a_{ij} \epi)(1+a_{ij}^2)^{-1/2}$,
 where $\epi$ and $\{\epi_{ij}\}^{[p/3]}_{j=1}$ are independently and identically distributed with $N(0,1)$, $\{\epi_{ij}\}^{[2p/3]}_{j=[p/3]+1}$ are independently and identically distributed following a double exponential distribution with location and scale parameters to be 0 and 1, and $\{\epi_{ij}\}^{p}_{j=[2p/3]+1}$ are independently and identically distributed following a mixture normal distribution with equal weights on $N(-1,1)$ and $N(1,0.5)$. The $\{a_{ij}\}_{j=1}^q$ are independently and identically distributed with $N(0,1)$ for the first $q$ variables and $a_{j}$=0 for $j\geq q$. The true feature set is $\mM_T=\{1,...,s\}$ with $s=20$.  Define $\theta^\top=(\theta_j)=\mathbf{1}_{[s/5]}\otimes(1,-1.1,1.2,-1.3,1.4)/2$, where $\mathbf{1}_b\in\mR^b$ is a vector with all elements equal to 1, and $\otimes$ denotes the Kronecker product. The response $Y_i$ is generated by a standard logistic regression. The feature dimension $p$ and total sample size $N$ are set at 10$^4$ and 10$^5$, respectively. The number of workers is set to $M=20,40$, and $50$.

 Next, we follow \cite{Fan:Song:2010} and compute the marginal maximum likelihood estimator for each feature $j$ on the $m$-th worker as $\widehat{\theta}_{j,m}$. In most cases, this should be a biased estimate for $\theta_j$, but could be useful for variable screening. This leads to a total of $M$ marginal estimators $\widehat{\theta}_{j,m}$. These are then averaged as $\widetilde{\boldsymbol\theta}_j=M^{-1}\sum_m\widehat{\theta}_{j,m}$, which is an overall marginal estimator for $\theta_j$. We next obtain the estimated feature set as $\widetilde \mM=\{1\leq j\leq p:|\widetilde{\boldsymbol\theta}_j|\geq \gamma_n\}$, where $\gamma_n$ is appropriately selected such that $|\widetilde \mM|=\lceil n/\log(n)\rceil$; see \cite{Fan:Song:2010} for a more detailed discussion. Consequently, the condition (C6) for the DQN algorithm is automatically satisfied. Thereafter, the proposed DQN method can be readily applied to the dimension reduced problem with only the selected feature involved.

 To measure the performance of the distributed screening procedure and the DQN algorithm, we compute the coverage rate for the $r$-th ($1\leq r\leq R$) replication as CR$^{(r)}=|\widetilde \mM^{(r)}\bigcap\mM_T^{(r)}|/|\mM_T^{(r)}|$. Then, the overall coverage rate is given by $\mbox{CR}=R^{-1}\sum_r\mbox{CR}^{(r)}$. The other metrics used in Section 3.1 are also considered. The detailed results are given in Table \ref{tab:mse-sis-1}. From Table \ref{tab:mse-sis-1}, we find that the implemented screening procedure is screening consistent in the sense that all CR values are equal to 1. Furthermore, with a fixed $N$, we find that a larger $M$ always leads to a smaller $n$. This leads to a smaller screening feature set with size $\lceil n/\log(n)\rceil$. Consequently, fewer redundant features are included. This further results in even smaller $\log$(MSE) values. Lastly, for the DQN algorithms, a slightly larger number of stages (i.e., larger $K$) are required to obtain an estimator that is competitive with MLE. These results are consistent with our theoretical findings in Theorems 1 and 2 and Corollary 1.

\begin{table}[h]
\caption{Log(MSE) values and corresponding SD, IQR, and range for ultrahigh-dimensional case. The numerical performance is evaluated for different $M$ ($\times 10$) and methods. The whole sample size $N$ and feature dimension $p$ are fixed at $ 10^5$ and $ 10 ^4$, respectively. Finally, the reported results are averaged based on $R=100$ simulations.}\label{tab:mse-sis-1}
\centering
\begin{spacing}{1}
\setlength{\tabcolsep}{0.8mm}{
\begin{tabular}{ccc|cc|cc|cc|cc|cc|c}
\hline
\hline
& & &\multicolumn{2}{c|}{Stage 0} & \multicolumn{2}{c|}{Stage 1} & \multicolumn{2}{c|}{Stage 2} & \multicolumn{2}{c|}{Stage 3} & \multicolumn{2}{c|}{Stage 4} &\\
&$M$ &CR&SR1 &BFGS &SR1 &BFGS&SR1 &BFGS&SR1 &BFGS&SR1 &BFGS&MLE \\[0.1em]
\hline
log&2&1.00&0.34&0.34&-1.01&-0.95&-0.93&-0.86&-0.99&-1.01&-1.11&-1.08&-1.00\\
(MSE)&4&1.00&0.35&0.35&-1.34&-0.87&-1.29&-1.11&-1.42&-1.42&-1.47&-1.49&-1.45\\
&5&1.00&0.40&0.40&-1.37&-0.73&-1.38&-1.13&-1.56&-1.53&-1.61&-1.64&-1.60\\
\multicolumn{14}{c}{} \\
SD&2&1.00&0.05&0.05&0.05&0.04&0.05&0.05&0.05&0.05&0.05&0.05&0.05\\
&4&1.00&0.06&0.06&0.06&0.05&0.07&0.07&0.06&0.06&0.06&0.06&0.06\\
&5&1.00&0.07&0.07&0.06&0.08&0.08&0.09&0.07&0.07&0.07&0.07&0.07\\
\multicolumn{14}{c}{} \\
IQR&2&1.00&0.08&0.08&0.07&0.05&0.06&0.08&0.07&0.07&0.06&0.06&0.07\\
&4&1.00&0.09&0.09&0.09&0.08&0.10&0.10&0.10&0.10&0.09&0.09&0.09\\
&5&1.00&0.09&0.09&0.08&0.09&0.12&0.13&0.10&0.11&0.10&0.09&0.10\\
\multicolumn{14}{c}{} \\
range&2&1.00&0.25&0.25&0.24&0.19&0.24&0.25&0.22&0.22&0.24&0.24&0.25\\
&4&1.00&0.27&0.27&0.26&0.25&0.34&0.36&0.30&0.32&0.30&0.30&0.29\\
&5&1.00&0.30&0.30&0.30&0.37&0.38&0.46&0.34&0.37&0.34&0.33&0.32\\
\hline
\end{tabular}}
\end{spacing}
\end{table}

\csubsection{Real Data Analysis}
\noindent

 In this section, we apply the proposed method to the THU Chinese
news dataset for illustration. The dataset is publicly available at
 \url{http://thuctc.thunlp.org}. The dataset consists of 14 types of Chinese news collected from Sina news (\url{https://news.sina.com.cn}) from 2005 to 2011.

 For the purpose of illustration, we generate response $Y_i$ as follows. We first select all the news of type {\it technology} and define the response $Y_i=1$. This leads to a total of $N_p=162,929$ positive cases. We next randomly sample a total of $\lceil1.5N_p\rceil$ negative cases without replacement from the other types of news. The corresponding response $Y_i$ is defined to be 0. Then, the total sample size is given by $N=407,322$. Different words are then extracted from the original documents. Those words with top $\mathscr{F}$\% frequencies for each class are selected.  They are then coded as binary \%covariates. We consider $\mathscr{F}\%=0.3\%$, 0.4\%, and 0.5\%, which leads to $p=998$, 1333, and 1660, respectively.
All data are then randomly shuffled and distributed to $M=20$ worker computers. The competing methods and performance measures remain the same as those in Section 3.2 but with $K=5$. Because we do not know the ground truth in real data analysis, the global estimators are then treated as if they were the true parameters. The experiment is randomly replicated $R=10$ times for a reliable evaluation. The results are summarized in Table \ref{tab:real} and Figure \ref{fig:real}.

From Table \ref{tab:real}, we find that the proposed DQN method has the lowest computation cost $T_1$. It outperforms other competing methods significantly in terms of computation efficiency. The computation advantage is particularly apparent when the feature dimension $p$ is relatively large.
 Moreover, we find that the communication cost $T_2$ of the DQN methods is slightly higher than the smallest $T_2$ value for the DAQN method. However, the overall time cost of DQN (i.e., $T$) remains the smallest. This suggests that the proposed DQN methods are computationally very competitive. From Figure \ref{fig:real}, we find that the proposed DQN methods also outperform their competitors slightly in terms of estimation accuracy with the smallest log(MSE) values.

\begin{table}
\caption{{Averaged computation cost $T_1$, communication cost $T_2$ and total time cost $T$ for the THU Chinese news dataset. The time cost is evaluated for different methods with different feature dimensions $p$.  The whole sample size $N$ and number of workers $M$ are fixed to be $407,322$ and $20$, respectively. The reported results are averaged for $R=10$ simulation replications. }}\label{tab:real}
\renewcommand\arraystretch{1.5}
\centering
\begin{spacing}{0.8}
\setlength{\tabcolsep}{1.4mm}{
\begin{tabular}{rr|rrrccc}
\hline
\hline
&$p$ &DOSN &CSL &DMGD&DAQN &DQN-BFGS&DQN-SR1 \\
\hline
$T_1$ & 998&1.22 &3.54&104.93&13.34&0.72&0.95\\
& 1333&2.14 &5.44&126.40&20.59&1.29&1.66\\
& 1660&2.63 &8.69&151.70&28.84&1.96&2.39\\
\hline
$T_2$ & 998&62.56 &0.36&38.52&0.12&0.99&0.71\\
& 1333&129.28 &0.45&54.27&0.26&1.26&0.92\\
& 1660&193.91 &0.56&66.65&0.36&1.57&1.10\\
\hline
$T$  & 998&63.78&3.90&143.45&13.46&1.71&1.66\\
& 1333&131.42&5.88&180.67&20.84&2.55&2.58\\
& 1660&196.54&9.25&218.36&29.20&3.53&3.49\\
\hline
\end{tabular}}
\end{spacing}
\end{table}

\begin{figure}[h]
\centering
\subfigure{\includegraphics[width=4.3in]{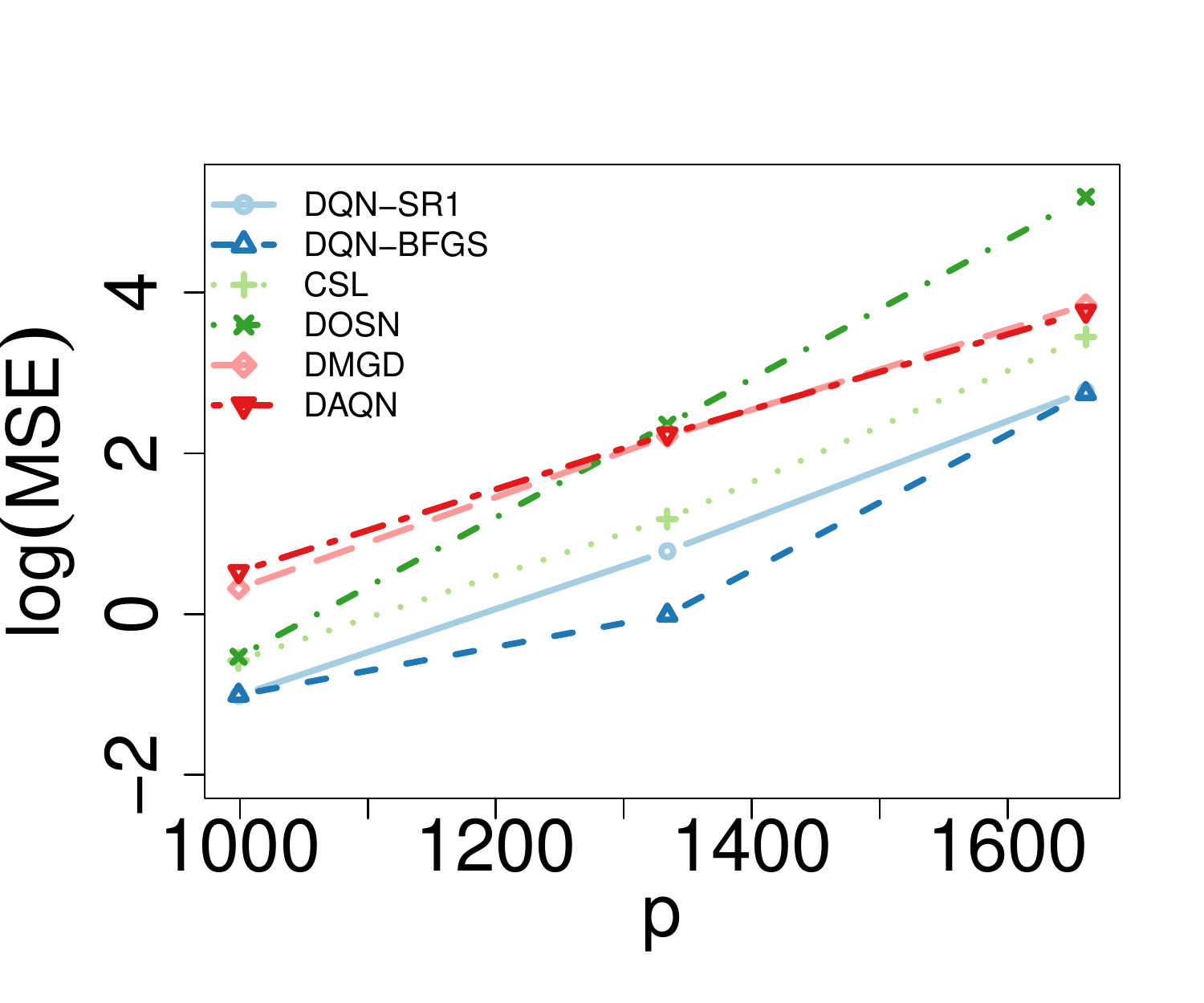} }
\caption{Log(MSE) values for the THU Chinese news dataset. The log(MSE) values are evaluated for different methods with different dimension $p$. The whole sample size $N$ is fixed ta $N = 407,322$ and number of workers $M$ is fixed at $M=20$. Finally, the reported log(MSE) values are averaged for $R=10$ simulations.}
\label{fig:real}
\end{figure}

\csection{CONCLUDING REMARKS}
\noindent

This article focuses on the discussion of statistical properties of DQN algorithms, which is motivated by two well-known quasi-Newton algorithms, i.e., SR1 and BFGS. The proposed algorithms are highly efficient both communicationally and computationally. We theoretically show that under mild conditions only a small number of iterations are needed to obtain an estimator as statistically efficient as the global one. As far as we know, this is the first work to discuss the statistical properties of the DQN methods.
Extensive numerical studies conducted on both simulation and real datasets are presented to illustrate the finite sample performance. To conclude this work, we discuss some interesting topics for future study. First, the DQN method proposed here requires that data among different worker computers are homogenous. This requirement may be difficult to satisfy for certain applications. Therefore, solving this problem
should be an exciting topic for future research. In addition, the proposed algorithm ignores privacy in inter-computer communication. This could be of great concern when sensitive information needs to be transferred. Conducting DQN while ensuring data privacy will be investigated in the future.

\noindent
\textbf{DATA AVAILABILITY STATEMENT.}
The datasets were derived from sources in the public domain: the official website of THU Chinese Text Classification Package (\url{http://thuctc.thunlp.org}).

\scsection{APPENDIX}

\renewcommand{\theequation}{A.\arabic{equation}}
\setcounter{equation}{0}
\noindent
\scsection{Appendix A: Proof of the Main Theoretical Results}
\noindent

For simplicity, we define the following notation in the proof.
Define $y_t = \dot{\mL}(\thetasptt) - \dot{\mL}(\thetaspt)$, and $s_t = \thetasptt - \thetaspt$ for $t \geq 1$; particularly, $y_0 = \dot{\mL}(\thetaos) - \dot{\mL}(\thetap)$ and $s_0 = \thetaos - \thetap$. In addition, for the BFGS updating, according to (\ref{eq:BFGS}), define
$
H_{t+1} = V_t^\top H_t V_t + \rho_t s_t s_t ^\top ,
$
where $V_t = I_p - \rho_t y_t s_t^\top $ and $\rho_t = 1/\big( s_t^\top y_t \big)$. For SR1 updating, according to (\ref{eq:rankone}), define
$
H_{t+1} = H_t + ( v_t^\top y_t )^{-1} ( v_t v_t^\top ),
$
for $t \geq 1$,
where $v_t = s_t - H_t y_t$. Particularly, $H_0 = M^{-1} \sum_{m=1}^M H_{(m,0)}$.

\scsubsection{A.1: Proof of Theorem 1}
\noindent

We decompose the theorem proof into two parts. In the first part, we show that the distance between $\thetaos$ and $\thetaglo$ is bounded by $\kappa \Big( M^{-1} \sum_{m=1}^M \Big[ \|\thetam - \theta_0\|^2 + \| \hesmtrue - \Omega(\theta_0) \|_2^2 + \big\|\big\{ \dddot{\mL}_{(m)}(\theta_0) - \dot{\Omega}(\theta_0) \big\} \big\{ (\thetam - \theta_0) \otimes I_p \big\} \big\|_2\Big] + \| \thetap - \theta_0 \|  \Big) \big \| \thetap - \theta_0 \big\|$ with probability tending to 1.
In the second part, we verify that when $N (p \log p)^2 / n^4 \to 0$, then $\| \thetaos - \thetaglo \| = o_p(N^{-1/2})$.

{\sc Part 1.} To analyze $\|\thetaos - \thetaglo\|$, we first define the following ``good events'':
\beqr
\label{apb:good-event}
\mE_0 &=& \bigg\{ \| \thetaglo - \theta_0 \| \leq \frac{\tau_{\min}}{4 C_{\max}} \bigg\} \\
\mE_m &=& \bigg\{ \| \thetam - \theta_0 \| \leq \min\Big\{ \frac{\tau_{\min}}{4 C_{\max}}, \delta \Big\} ,n^{-1} \sum_{i \in \mS_m} C(X_i,Y_i) \leq 2 C_{\max} , \nonumber\\
&&  \| \ddot{\mL}_{(m)} (\theta_0) - \Omega(\theta_0) \|_2 \leq \frac{\delta \tau_{\min}}{ 4}, \| \dot{\mL}_{(m)} (\theta_0) \| \leq \frac{ (1- \delta) \tau_{\min} \delta_{\min} }{4}  \bigg\} , \nonumber
\eeqr
where $\delta_{\min} = \min \big\{ \delta, \delta \tau_{\min} /(4 C_{\max}) \big\}$.
By Lemma \ref{lm:bad-event}, we know $P(\bigcup_{m=0}^M \mE_m^c ) \to 0$. In addition, it could be verified that the events $\bigcap_{m=0}^M\mE^{\prime}_m$ defined in Lemma \ref{lm:new-event} holds under $\bigcap_{m=0}^M \mE_m $. Thus, it suffices to analyze the upper bound of $\|\thetaos - \thetaglo\|$ under $\bigcap_{m=0}^M \mE_m $ and $\bigcap_{m=0}^M\mE^{\prime}_m$.

We then proceed to study {\sc Part 1}. Recall the definition of $ \thetaos$, then by (C5), we have
$\thetaos = \thetap - M^{-1} \sum_{m=1}^M \big\{ \hesmp \big\}^{-1} \dot{\mL} (\thetap) $
.
In addition, define $\thetanr = \thetap - \big\{\ddot{\mL}(\thetap) \big\}^{-1} \dot{\mL}(\thetap)$ to represent the one-step Newton--Raphson estimator. Then,
by the triangle inequality,
we have $\|\thetaos - \thetaglo \| \leq$
\beqrs
&& \| \thetaos - \thetanr \| + \| \thetanr - \thetaglo \| \\
&=& \Big\| \thetap - M^{-1} \sum_{m=1}^M \big\{ \hesmp \big\}^{-1} \dot{\mL} (\thetap) \\
&& - \Big[ \thetap - \big\{\ddot{\mL}(\thetap) \big\}^{-1} \dot{\mL}(\thetap) \Big] \Big\| + \| \thetanr - \thetaos \| \\
&=& \bigg\| \Big[ \big\{\ddot{\mL}(\thetap) \big\}^{-1} - M^{-1} \sum_{m=1}^M \big\{ \hesmp \big\}^{-1} \Big]\dot{\mL} (\thetap) \bigg\| + \| \thetanr - \thetaglo \|.
\eeqrs
Denote $\Delta_1 =  \big\{\ddot{\mL}(\thetap) \big\}^{-1} - M^{-1} \sum_{m=1}^M \big\{ \hesmp \big\}^{-1}$. We then investigate $\Delta_1$, $\thetanr- \thetaglo$, and $\dot{\mL} (\thetap)$ in the following three steps, respectively.

\emph{Step 1.}
By the triangle inequality, we have
$
\|\Delta_1 \|_2 \leq \big \| \Hint - \Htrue \big\|_2 + \big\| \Htrue - \Havgtrue \big\|_2 + \big\| M^{-1} \sum_{m=1}^M \big\{ \hesmtrue \big\}^{-1} - M^{-1} \sum_{m=1}^M $ $\big\{ \ddot{\mL}_{(m)}(\thetam) \big\}^{-1} \big\|_2 := \big\|\Delta^{(1)}_1\big\|_2 + \| \Delta^{(2)}_1\|_2 + \|\Delta^{(3)}_1 \|_2.
$
We proceed to calculate the three terms separately.

\emph{Step 1.1.} First, for any matrix $B$, we have $\|(B+\Delta B)^{-1} - B^{-1} \|_2 \leq \|B^{-1} \|_2^2 \| \Delta B\|_2$ \citep{jordan2019communication}. Substituting $B = \hestrue $ and $\Delta B = \hesint - \hestrue$, it could be shown that
\beqrs
\label{apb:deltaone}
\|\Delta_1^{(1)} \|_2 &\leq & \Big \| \Htrue \Big\|_2^2 \Big\| \hesint - \hestrue \Big\|_2 \leq \frac{4}{(1-\delta)^2 \tau_{\min}^2} 2 C_{\max} \| \thetap - \theta_0 \|.
\eeqrs
The second inequality holds because
$
\big\| \Htrue \big \|_2 \leq 2/\big\{(1-\delta) \tau_{\min} \big\}
$
under $\bigcap_{m=0}^M\mE^\prime_m$
and
$
\| \hesint - \hestrue \|_2 \leq 2 C_{\max} \| \thetap - \theta_0 \|
$
under $\bigcap_{m=0}^M \mE_m$.
Consequently, there exists a constant $\kappa > 0$ such that
$
\|\Delta_{1}^{(1)} \|_2 \leq \kappa \| \thetap - \theta_0 \|/(6 \times 2 C_{\max}) .
$

\emph{Step 1.2}
Next, we analyze $\Delta_1^{(2)}$. It could be shown that $\Delta_{1}^{(2)} =$
\beqr
&&M^{-1} \sum_{m=1}^M \big\{ \ddot{\mL}_{(m)} (\theta_0) \big\}^{-1} \big\{  \hestrue - \ddot{\mL}_{(m)} (\theta_0) \big\} \big\{\hestrue\big\}^{-1}\nonumber \\
&=& M^{-1} \sum_{m=1}^M \Big( \big[ \big\{\hesmtrue \big\}^{-1} - \big\{ \hestrue \big\}^{-1} \big] \big\{ \hestrue - \hesmtrue \big\} \big\{ \hestrue \big\}^{-1} + \nonumber \\
&&\big\{ \hestrue \big\}^{-1} \big\{ \hestrue - \hesmtrue \big\} \big\{ \hestrue \big\}^{-1} \Big) \nonumber\\
&=& M^{-1} \sum_{m=1}^M \big\{ \ddot{\mL}_{(m)} (\theta_0) \big\}^{-1}  \big\{  \hestrue - \ddot{\mL}_{(m)} (\theta_0) \big\} \big\{\hestrue\big\}^{-1}  \big\{  \hestrue - \ddot{\mL}_{(m)} (\theta_0) \big\} \big\{\hestrue\big\}^{-1}. \nonumber
\eeqr
We then have $\|\Delta_{1}^{(2)} \|_2 \leq \| \big\{\hesmtrue\big\}^{-1} \|_2 \| \Htrue \|_2^2 \times M^{-1} \sum_{m=1}^M \big\| \hestrue - $ $\hesmtrue \big\|_2^2 \leq 6 \big\{ (1 - \delta)^3 \tau_{\min}^3 \big\}^{-1} \times 8 M \sum_{m=1}^M \big\| \hesmtrue - \Omega(\theta_0) \big\|_2^2 $.
This is because
$
\big\| \hestrue - \hesmtrue \big\|_2^2 \leq 2 \Big\{ \big\| \hestrue - \Omega(\theta_0) \|_2^2 + \| \hesmtrue - \Omega(\theta_0) \big\|_2^2 \Big\}
$
and
$
\big\| \hestrue - \Omega(\theta_0) \|_2^2 \leq
\big\| M^{-1} \sum_{m=1}^M \big(\hesmtrue - \Omega(\theta_0) \big) \big\|_2^2 \leq (1 +1/M) M^{-1} \sum_{m=1}^M  \| \hesmtrue - \Omega(\theta_0) \big\|_2^2
$.
Therefore, there exists a constant $\kappa > 0$ such that

$$
\big\|\Delta_{1}^{(2)} \big\|_2 \leq \frac{\kappa}{6 \times 2 C_{\max}} M^{-1} \sum_{m=1}^M \| \hesmtrue - \Omega(\theta_0) \|_2^2.
$$

\emph{Step 1.3}. Moreover, it could be proved that
$$
\Delta_{1}^{(3)} = M^{-1} \sum_{m=1}^M \big\{\hesmp\big\}^{-1} \Big\{\hesmp - \hesmtrue \Big\}  \big\{\hesmtrue \big\}^{-1}.
$$
By Taylor's expansion, Cauchy--Schwarz inequality, and \emph{Step 1.2}, we have
\beqr
\label{apa:delta1-3}
\|\Delta_{1}^{(3)} \|_2 &\leq& \Big\| M^{-1} \sum_{m=1}^M \big\{\hestrue\big\}^{-1} \Big\{\hesmp - \hesmtrue \Big\} \big\{\hestrue\big\}^{-1}  \Big\|_2 \nonumber\\
&&+ \frac{\kappa}{6 \times 2 C_{\max}} M^{-1} \sum_{m=1}^M \Big\{ \| \hesmtrue - \Omega(\theta_0) \|_2^2 + \| \thetam -\theta_0 \|^2 \Big\}.
\eeqr
Hence, it suffices to study the first term of (\ref{apa:delta1-3}). Using Taylor's expansion again, it could be verified that $\hesmp - \hesmtrue =$
\beqrs
&&\dddot{\mL}_{(m)}(\theta_0) \Big\{ (\thetam - \theta_0) \otimes I_p \Big\} + \big\{ \dddot{\mL}_{(m)}(\xi_{(m)}) - \dddot{\mL}_{(m)}(\theta_0) \big\} \big\{ (\thetam - \theta_0) \otimes I_p \big\} \\
&=& \Big\{ \dddot{\mL}_{(m)}(\theta_0) - \dot{\Omega}(\theta_0) \Big\} \Big\{ (\thetam - \theta_0) \otimes I_p \Big\} + \dot{\Omega}(\theta_0) \Big\{ (\thetam - \theta_0) \otimes I_p \Big\}  + \mO,
\eeqrs
where $\xi_{(m)} = \eta_{(m)} \thetam + (1-\eta_{(m)}) \theta_0$ for some $0 \leq \eta_{(m)} \leq 1$, $\mO = \big\{ \dddot{\mL}_{(m)}(\xi_{(m)}) - \dddot{\mL}_{(m)}(\theta_0) \big\} \big\{ (\thetam - \theta_0) \otimes I_p \big\}$. In addition, we have $\|\mO\|_2 \leq 2C_{\max} \|\thetam - \theta_0 \|^2$ by (C4).
Replacing the results back into (\ref{apa:delta1-3}), we obtain
$
\|\Delta_{1}^{(3)} \|_2 \leq \kappa \Big(  M^{-1} $ $ \sum_{m=1}^M \Big[ \|\thetam - \theta_0\|^2 +  \| \hesmtrue - \Omega(\theta_0) \|_2^2 + \big\| \big\{ \dddot{\mL}_{(m)}(\theta_0) - \dot{\Omega}(\theta_0) \big\} \big\{ (\thetam - \theta_0) \otimes I_p \big\} \big\|_2 \Big] + \| \thetap - \theta_0 \|  \Big) /(6 \times 2 C_{\max}).
$
Combining the above-mentioned results, we have
$
\|\Delta_{1}\|_2 \leq \kappa \Big(  M^{-1} \sum_{m=1}^M \Big[ \|\thetam - \theta_0\|^2 +  \| \hesmtrue - \Omega(\theta_0) \|_2^2 + \big \|\big\{ \dddot{\mL}_{(m)}(\theta_0) - \dot{\Omega}(\theta_0) \big\} \big\{ (\thetam - \theta_0) \otimes I_p \big\} \big\|_2 \Big] + \| \thetap - \theta_0 \|  \Big) /(4 C_{\max}).
$
This finishes the proof of \emph{Step 1.}

\emph{Step 2.} In this step, we study $\thetanr- \thetaglo$ and $\dot{\mL} (\thetap)$.
From Theorem 5.3 in \cite{bubeck2015theory}, when $\|\thetap - \thetaglo \| \leq \lambda_{\min} \big\{ \ddot{\mL} (\thetaglo) \big\} / \big(2 C_{\operatorname{ge}} \big)$, where $C_{\operatorname{ge}}$ is the global Lipschitz constant of $\hestheta$ such that $\| \ddot{\mL} (\theta^{\prime}) - \ddot{\mL} (\theta^{\prime\prime})  \| \leq C_{\operatorname{ge}} \|\theta^{\prime} - \theta^{\prime\prime}\|$, we have
\beq
\label{apa:nr-eq}
\|\thetanr - \thetaglo \| \leq \frac{C_{\operatorname{ge}} } {\lambda_{\min} \big\{ \ddot{\mL} (\thetaglo) \big\} } \| \thetap - \thetaglo \|^2 \leq \frac{4 C_{\max}}{ (1-\delta) \tau_{\min} } \| \thetap - \thetaglo \|^2.
\eeq
Moreover, by (C4), it could be verified that
$
\|\dot{\mL} (\thetap) - \dot{\mL} (\thetaglo) \| \leq 2C_{\max} $ $\| \thetap - \thetaglo \|.
$
This finishes the proof of \emph{Step 2.}

Combining the results of \emph{Steps 1} and \emph{ 2,} we have
$
\| \thetaos - \thetaglo \| \leq \kappa \Big(  M^{-1} \sum_{m=1}^M $ $ \Big[ \|\thetam - \theta_0\|^2 +  \| \hesmtrue - \Omega(\theta_0) \|_2^2 + \big\|\big\{ \dddot{\mL}_{(m)}(\theta_0) - \dot{\Omega}(\theta_0) \big\} \big\{ (\thetam - \theta_0) \otimes I_p \big\}\big\|_2 \Big]  + \| \thetap - \theta_0 \| + \| \thetaglo - \theta_0 \|  \Big) \big \| \thetap - \thetaglo \big\|
$
with probability tending to 1. Noting that $\|\thetaglo - \theta_0\|$ is a negligible higher-order term, and we finish the first part.

{\sc Part 2.} To prove the second part, we separately analyze the convergence properties of $M^{-1} \sum_{m=1}^M \|\thetam - \theta_0\|^2$, $ M^{-1} \sum_{m=1}^M   \| \hesmtrue - \Omega(\theta_0) \|_2^2$, $M^{-1} \sum_{m=1}^M \big\{ \dddot{\mL}_{(m)}(\theta_0) - \dot{\Omega}(\theta_0) \big\} \big\{ (\thetam - \theta_0) \otimes I_p \big\}$ and $\|\thetap - \theta_0\|^2$.
By Lemma \ref{lm:est-expectation} we know
\beqrs
\label{apa:expect-order}
E \| \thetam - \theta_0 \|^2 &\leq& C_1 n^{-1} C_G^2\big\{ 1 + o(1) \big\} \\
E \Big\{ \| \hesmtrue - \Omega(\theta_0)  \|_2^2 \Big\} &\leq&  C_2 \frac{\log p}{n} \big\{ 1 + o(1) \big\} \nonumber \\
E \Big\| \big\{ \dddot{\mL}_{(m)}(\theta_0) - \dot{\Omega}(\theta_0) \big\} \big\{ (\thetam - \theta_0) \otimes I_p \big\}\Big\|_2 &\leq&  C_4 \frac{ p \sqrt{\log p}}{n} \big\{ 1 + o(1) \big\} \nonumber \\
E\big[ \| \thetap - \theta_0 \|^2  \big] &\leq& \bigg( \frac{2 C_G^2}{ \tau_{\min}^2 N } + \frac{C_3 C_G^2 C_H^2 \log p}{ \tau_{\min}^4 n^2 }\bigg) \big\{ 1 + o(1)\big\} \nonumber
\eeqrs
for some positive constants $C_1$--$C_4$. Moreover,
by Markov's inequality, we have
\begin{gather}
M^{-1} \sum_{m=1}^M \Big\{ \|\thetam - \theta_0\|^2 +  \| \hesmtrue - \Omega(\theta_0) \|_2^2 \Big\} + \big\|\big\{ \dddot{\mL}_{(m)}(\theta_0) - \dot{\Omega}(\theta_0) \big\} \big\{ (\thetam - \theta_0) \otimes I_p \big\}\big\|_2  \nonumber\\
+ \| \thetap - \theta_0 \|  = O_p( n^{-1} p \sqrt{\log p}+ N^{-1/2}), \text{ and } \|\thetap - \theta_0\| = O_p(1/\sqrt{N} + \sqrt{\log p}/n).
\label{apa:one-stage-order}
\end{gather}
Hence, $\|\thetaos - \thetaglo\| = O_p(n^{-2} p \log p ) + o_p(N^{-1/2})$. Furthermore, under the condition $N( p \log p)^2/n^4 $ $\to 0$, we have
$
N^{1/2} \| \thetaos - \thetaglo \|  \to_p 0,
$
which finishes the proof of the second part, thereby completing the proof of the entire theorem.

\scsubsection{A.2: Proof of Theorem 2}
\noindent

To verify Theorem \ref{thm:one-step-updating}, we first prove that $\big\| \thetatwo - \thetaglo \big\| \leq \kappa_2  \Big(  M^{-1} \sum_{m=1}^M \Big[ \|\thetam - \theta_0\|^2 +  \| \hesmtrue - \Omega(\theta_0) \|_2^2 + \big\|\big\{ \dddot{\mL}_{(m)}(\theta_0) - \dot{\Omega}(\theta_0) \big\} \big\{ (\thetam - \theta_0) \otimes I_p \big\}\big\|_2 \Big] + \| \thetap - \theta_0 \|  \Big) \big \| \thetaos - \thetaglo \big\|$ for some constant $\kappa_2>0$, with probability tending to 1. Next, we verify the optimality of $\thetatwo$ under the condition $Np^4 (\log p)^3/n^6\to 0$.

Note that by Algorithms \ref{alg:two-step} and \ref{alg:SR1-two-step}, the proposed methods realize the global update of the approximated Hessian inverse. In other words, the two-stage estimator update (\ref{eq:bfgs-two-step}) is equal to $\thetatwo = \thetaos - H_1 \dot{\mL}(\thetaos)$.
For convenience, instead of directly studying $\thetaos - H_1 \dot{\mL}(\thetaos)$, we investigate $\thetaos - (B_1)^{-1} \dot{\mL}(\thetaos)$, where $B_1 = H_1^{-1}$. By the triangle inequality, it could be verified that
$$
\|\thetatwo - \thetaglo\| \leq \| \thetaos - \big\{\ddot{\mL}(\thetaglo) \big\}^{-1} \dot{\mL}(\thetaos) - \thetaglo \| + \Big\| \Big[ \big\{\ddot{\mL}(\thetaglo) \big\}^{-1} - B_1^{-1} \Big] \dot{\mL}(\thetaos)  \Big\| .
$$
We denote $\Delta_2^{(1)} = \thetaos - \big\{\ddot{\mL}(\thetaglo) \big\}^{-1} \dot{\mL}(\thetaos) - \thetaglo $ and $\Delta_2^{(2)} = \Big[ \big\{\ddot{\mL}(\thetaglo) \big\}^{-1} - B_1^{-1} \Big] \dot{\mL}(\thetaos)$, where $\Delta_2^{(1)}$ is independent of SR1 or BFGS update. Furthermore, by similar analytical techniques as those used in (\ref{apa:nr-eq}), we have
$
\|\Delta_2^{(1)} \| \leq \kappa_2^{\prime} \| \thetaos - \thetaglo \|^2
$
for some constant $\kappa_2^{\prime} > 0$ with probability tending to 1. Hence, it suffices to study $\Delta_2^{(2)}$. Therefore, we investigate $\Delta_2^{(2)}$ under the good events $\bigcap_{m=0}^M \mE_m $ and $\mE^{\prime}$ by SR1 and BFGS update separately.

{\sc Part 1 (SR1).} By the Sherman--Morrison formula \cite[Theorem 10.8]{burden2015numerical}, the SR1 updating formula can be expressed as
$$
B_1 = B_0 + \frac{ (y_0 - B_0 s_0 )(y_0 - B_0 s_0)^\top }{ (y_0 - B_0 s_0)^\top y_0},
$$
where $B_0 = H_0^{-1}$. Then, we proceed to study $\Delta_2^{(2)}$, which can be rewritten as
\beq
\Delta_2^{(2)} = \big\{ \ddot{\mL}(\thetaglo) \big\}^{-1} \big\{ B_1 - \ddot{\mL}(\thetaglo) \big\} B_1^{-1} \dot{\mL}(\thetaos).
\eeq
Further, when $s_1 = B_1^{-1} \dot{\mL}(\thetaos)$, by the triangle inequality, we have
$$
\|\Delta_2^{(2)}\| \leq \| \big\{ \ddot{\mL}(\thetaglo) \big\}^{-1} \|_2 \big\{ \|y_1 - B_1 s_1\| + \|y_1 - \ddot{\mL}(\thetaglo) s_1 \| \big\}.
$$
We then study $y_1 - B_1 s_1$ and $y_1 - \ddot{\mL}(\thetaglo) s_1$.

\emph{Step 1.} First, we investigate $y_1 - B_1 s_1$. By (\ref{eq:rankone}), we have
$$
y_1 - B_1 s_1 = y_1 - B_0 s_1 + \frac{r_0 r_0^\top s_1}{ r_0^\top s_0},
$$
where $r_t = y_t - B_t s_t$ for any $t > 0$.
Then by Taylor's expansion, it could be proved that
\beqr
\label{apc:r1-1}
\|y_1 - B_0 s_1 \| &\leq& \Big\| ( \ddot{\mL}(\thetaos) - B_0 )(\thetatwo - \thetaos) + \ddot{\mL}(\xi_1) (\thetatwo - \thetaos) \nonumber \\
&& - \ddot{\mL}(\thetaos) (\thetatwo - \thetaos) \Big\| \nonumber\\
&\leq& 2 \Big\{ \| \ddot{\mL}(\thetaos) - B_0 \|_2 \| \thetaos - \thetaglo \| + \| \ddot{\mL}(\xi_1) - \ddot{\mL}(\thetaos) \|_2 \| \thetaos - \thetaglo\| \Big\} \nonumber\\
&\leq & 2 \Big\{ \| \ddot{\mL}(\thetaos) - B_0 \|_2 + 2C_{\max} \| \thetaos - \thetaglo \| \Big\}\| \thetaos - \thetaglo \|.
\eeqr
Here $\xi_1 = \eta_1 \thetatwo + (1 - \eta_1) \thetaos$ with some $0 \leq \eta_1 \leq 1$. The second inequality in (\ref{apc:r1-1}) holds by the triangle inequality $\| \thetatwo - \thetaos \| \leq \| \thetatwo - \thetaglo \| + \| \thetaos - \thetaglo \|$ and Lemma \ref{lm:linearly-converge}. The last inequality in (\ref{apc:r1-1}) holds by (C4).
In addition, from (\ref{apd:well-define}), it could be verified that $|( r_0^\top s_0)^{-1} (r_0 r_0^\top s_1)| \leq $
\beqr
\label{apc:r1-2}
&& \frac{ \|r_0 \| \| s_1 \|  }{c_1 \|s_0\|} = \frac{ \| \dot{\mL}(\thetaos) - \dot{\mL}(\thetap) - B_0 (\thetaos - \thetap) \| \| \thetatwo - \thetaos \| }{c_1 \| \thetaos - \thetap \|} \nonumber\\
&=& \frac{ \|\big\{ \ddot{\mL}(\thetaos) - B_0  \big\} s_0 +\big\{ \ddot{\mL} (\xi_0) - \ddot{\mL}(\thetaos) \big\} s_0 \| \| \thetatwo - \thetaos \| }{c_1 \| \thetaos - \thetap \|} \nonumber\\
&\leq& 2 \Big\{ \| \ddot{\mL}(\thetaos) - B_0 \|_2 \| \thetaos - \thetaglo \| + 4 C_{\max} \| \thetaos - \thetaglo \| \| \thetap - \thetaglo \|\Big\}.
\eeqr
Here $\xi_0 = \eta_0 \thetaos + (1 - \eta_1) \thetap$ with some $0 \leq \eta_0 \leq 1$, and the last inequality holds by Lemma \ref{lm:linearly-converge} and (C4).
Combining the results of (\ref{apc:r1-1}) and (\ref{apc:r1-2}), we have
$
\| y_1 - B_1 s_1 \| \leq 4 \Big\{ \| \ddot{\mL}(\thetaos) - B_0 \|_2 + 4C_{\max} \|\thetap - \thetaglo \| \Big\} \| \thetaos - \thetaglo \|.
$
Furthermore,
$
\| \ddot{\mL}(\thetaos) - B_0 \|_2 \leq \| \ddot{\mL}(\thetaos) \|_2^2 \| \big\{ \ddot{\mL}(\thetaos)\big\}^{-1} - M^{-1} \sum_{m=1}^M \hesmp \|_2.
$
Using an analysis technique similar to the one used in Appendix A.1 \emph{Step 1} to study the value of $\Delta_1$, we have
$
\| \ddot{\mL}(\thetaos) - B_0 \|_2 \leq \kappa_2^{\prime} \Big(  M^{-1} \sum_{m=1}^M \Big[ \|\thetam - \theta_0\|^2 +  \| \hesmtrue - \Omega(\theta_0) \|_2^2 + \big\| \big\{ \dddot{\mL}_{(m)}(\theta_0) - \dot{\Omega}(\theta_0) \big\} \big\{ (\thetam - \theta_0) \otimes I_p \big\} \big\|_2 \Big] + \| \thetap - \theta_0 \|  \Big).
$
Hence, we have
$\| y_1 - B_1 s_1 \| \leq \kappa_2^{\prime} \Big(  M^{-1} \sum_{m=1}^M \Big[ \|\thetam - \theta_0\|^2 +  \| \hesmtrue - \Omega(\theta_0) \|_2^2 + \big\|\big\{ \dddot{\mL}_{(m)}(\theta_0) - \dot{\Omega}(\theta_0) \big\} \big\{ (\thetam - \theta_0) \otimes I_p \big\}\big\|_2 \Big] + \| \thetap - \theta_0 \|  + \|\thetaglo - \theta_0 \|  \Big) \| \thetaos - \thetaglo \|.
$
This accomplishes the proof of \emph{Step 1.}

\emph{Step 2.} In this step, we show that $y_1 - \ddot{\mL}(\thetaglo) s_1$ is a negligible higher-order term. By Taylor's expansion, it could be proved that
\beqr
\label{apc:step21}
\|y_1 - \ddot{\mL}(\thetaglo) s_1\| &=& \| \dot{\mL} (\thetatwo) - \dot{\mL} (\thetaos) - \ddot{\mL}(\thetaglo) (\thetatwo - \thetaos ) \| \\
&\leq & \| \big\{ \ddot{\mL}(\xi_1) - \ddot{\mL} (\thetaglo) \big\} \| \thetatwo - \thetaos \| \leq 4C_{\max} \| \thetaos - \thetaglo \|^2 . \nonumber
\eeqr
The last inequality holds by Lemma \ref{lm:linearly-converge} and (C4).
Combining the results of {\sc Steps 1} and {\sc 2}, we finish the first part of the theorm proof.

{\sc Part 2 (BFGS).} Recall
$
\Delta_2^{(2)} = \big\{ \ddot{\mL}(\thetaglo) \big\}^{-1} \big\{ B_1 - \ddot{\mL}(\thetaglo) \big\} B_1^{-1} \dot{\mL}(\thetaos) = \big\{ \ddot{\mL}(\thetaglo) \big\}^{-1} $ $\big\{ B_1 - \ddot{\mL}(\thetaglo) \big\} s_1.
$
Denote
$$
P_0 = I_p - \frac{\big\{\ddot{\mL}(\thetaglo)\big\}^{1/2}s_0 \Big[ \big\{\ddot{\mL}(\thetaglo)\big\}^{-1/2} y_0 \Big]^\top  }{y_0^\top s_0}.
$$
Then it could be shown by \citet[Lemma 5.1]{broyden1973local} that
\beqrs
E_1 = P_0^\top E_0 P_0 + \frac{\big\{ \ddot{\mL}(\thetaglo) \big\}^{-1/2} \big\{ y_0 - \ddot{\mL}(\thetaglo) s_0 \big\} \Big[ \big\{\ddot{\mL}(\thetaglo) \big\}^{-1/2} y_0\Big]^\top }{y_0^\top s_0} \\
+ \frac{\big\{ \ddot{\mL}(\thetaglo) \big\}^{-1/2} y_0 \big\{y_0 - \ddot{\mL}(\thetaglo) s_0\big\}^\top \big\{ \ddot{\mL}(\thetaglo) \big\}^{-1/2} P_0}{ y_0^\top s_0},
\eeqrs
where $E_0 = \big\{ \ddot{\mL}(\thetaglo) \big\}^{-1/2} \big\{ B_0 - \ddot{\mL}(\thetaglo) \big\} \big\{ \ddot{\mL}(\thetaglo) \big\}^{-1/2}$ and $E_1 = \big\{ \ddot{\mL}(\thetaglo) \big\}^{-1/2} \big\{ B_1 - \ddot{\mL}(\thetaglo) \big\} \big\{ \ddot{\mL}(\thetaglo) \big\}^{-1/2}$.
Then it could be proved that $\big\{ \ddot{\mL}(\thetaglo) \big\}^{1/2} \Delta_2^{(2)} =$
\beqrs
&&
E_1 \big\{ \ddot{\mL}(\thetaglo) \big\}^{1/2} s_1 =
P_0^\top E_0 P_0 \big\{ \ddot{\mL}(\thetaglo) \big\}^{1/2} s_1 + \frac{ \big\{ \ddot{\mL}(\thetaglo) \big\}^{-1/2} \big\{ y_0 - \ddot{\mL}(\thetaglo) s_0\big\} y_0^\top s_1  }{ y_0^\top s_0} \\
&&+ \frac{ \big\{ \ddot{\mL}(\thetaglo) \big\}^{-1/2}  y_0 \big\{ y_0 - \ddot{\mL}(\thetaglo) s_0 \big\}^\top \big\{ \ddot{\mL} (\thetaglo) \big\}^{-1/2} P_0 s_1 }{y_0^\top s_0} := \Delta_2^{(2,1)} + \Delta_2^{(2,2)} + \Delta_2^{(2,3)}.
\eeqrs
Note that  $\|y_0 - \ddot{\mL}(\thetaglo) s_0 \| \leq \| \thetap - \thetaglo \|^2$ by (\ref{apc:step21}). Applying similar analytical techniques as in {\sc Part 1},
and using $|y_0 s_0| = |s_0^\top \ddot{\mL}(\xi_0) s_0| \geq c_1 \|s_0\|^2$ for some positive constant $c_1 > 0$, it could be verified that
\beqrs
\| \Delta_2^{(2,2)}\| &\leq& \frac{\|\thetap - \thetaglo \|^2 \|y_0 \| \| s_1 \| }{ c_1 \| s_0 \|^2 } \leq \frac{ \|\thetap - \thetaglo \|^2 \| \ddot{\mL}(\xi_0)\|_2 \|s_1 \| }{c_1 \Big| \|\thetap - \thetaglo\| - \| \thetaos - \thetaglo \| \Big| } \\
&\leq& \kappa_2^{\prime} \| \thetap - \thetaglo \| \| \thetaos - \thetaglo \| .
\eeqrs
Similarly, for $\|P_0\| \leq 1 + 1/c_1$, we have $\| \Delta_2^{(2,3)}\| \leq \kappa_2^{\prime} \| \thetap - \thetaglo \| \| \thetaos - \thetaglo \|$.
Thus, it can be verified that
$
\| \Delta_2^{(2,1)}\| \leq \kappa_2^{\prime} \| B_0 - \ddot{\mL} (\thetaglo) \|_2 \| \thetaos - \thetaglo \| .
$
By similar analysis of $\| \ddot{\mL}(\thetap) - B_0 \|_2$ as in Appendix A.1 \emph{Step 1,} we have
$
\| B_0 - \ddot{\mL} (\thetaglo) \| \leq \kappa_2^{\prime} \Big(  M^{-1} \sum_{m=1}^M \Big[ \|\thetam - \theta_0\|^2 +  \| \hesmtrue - \Omega(\theta_0) \|_2^2 + \big\|\big\{ \dddot{\mL}_{(m)}(\theta_0) - \dot{\Omega}(\theta_0) \big\} \big\{ (\thetam - \theta_0) \otimes I_p \big\} \big\|_2 \Big] + \| \thetap - \theta_0 \|  + \|\thetaglo - \theta_0 \|  \Big).
$
Thus, $\Delta_2^{(2)}$ could be bounded by
$
\| \Delta_2^{(2)} \| \leq \kappa_2^{\prime} \Big(  M^{-1} \sum_{m=1}^M \Big[ \|\thetam - \theta_0\|^2 +  \| \hesmtrue - \Omega(\theta_0) \|_2^2 +\big\| \big\{ \dddot{\mL}_{(m)}(\theta_0) - \dot{\Omega}(\theta_0) \big\} \big\{ (\thetam - \theta_0) \otimes I_p \big\} \big\|_2 \Big] + \| \thetap - \theta_0 \|  + \|\thetaglo - \theta_0 \|  \Big) \| \thetaos - \thetaglo \|.
$
This finishes the proof of {\sc Part 2 (BFGS).}

Next, by (\ref{apa:one-stage-order}) again, we have
$
\|\thetatwo - \thetaglo\| = O_p\big( p^2 (\log p)^{3/2}/n^3 \big) + o_p\big(1/\sqrt{N}\big).
$
As a result, under the condition $ N( p^4 (\log p)^3)/n^6 $ $\to 0$, we have $\|\thetatwo - \thetaglo\| = o_p(N^{-1/2})$, which accomplishes the whole theorem proof.

\scsubsection{A.3: Proof of Corollary 1}
\noindent

First, to prove $\big\| \thetaK - \thetaglo \big\| \leq \kappa_{_{K}}  \Big(  M^{-1} \sum_{m=1}^M \Big[ \|\thetam - \theta_0\|^2 +  \| \hesmtrue - \Omega(\theta_0) \|_2^2 +\big\| \big\{ \dddot{\mL}_{(m)}(\theta_0) - \dot{\Omega}(\theta_0) \big\} \big\{ (\thetam - \theta_0) \otimes I_p \big\} \big\|_2\Big] + \| \thetap - \theta_0 \|  \Big)^K \big \| \thetap - \theta_0 \big\|$, we verify that
$\big\| \thetaK - \thetaglo \big\| \leq \kappa_{_{K}}  \Big(  M^{-1} \sum_{m=1}^M \Big[ \|\thetam - \theta_0\|^2 +  \| \hesmtrue - \Omega(\theta_0) \|_2^2 +\big\| \big\{ \dddot{\mL}_{(m)}(\theta_0) - \dot{\Omega}(\theta_0) \big\} \big\{ (\thetam - \theta_0) \otimes I_p \big\} \big\|_2\Big] + \| \thetap - \theta_0 \|  \Big) \big \| \thetakone - \thetaglo \big\|$.
In addition, by the triangle inequality, we have $\|\thetakm - \thetaglo \nonumber\| =$
\beqr
\label{apc:k-step}
&&\|\thetakmone - B_{k-1}^{-1} \dot{\mL}(\thetakmone) - \thetaglo\| \\
&\leq& \| \thetakmone - \big\{\ddot{\mL}(\thetaglo) \big\}^{-1}\dot{\mL} (\thetakmone) - \thetaglo \| + \| \big[ \big\{ \ddot{\mL} (\thetaglo)\big\}^{-1} - B_{k-1}^{-1} \big] \dot{\mL}(\thetakmone) \| \nonumber
\eeqr
for any $2 \leq k \leq K$. Similar to the analysis of (\ref{apa:nr-eq}) at the beginning of Appendix A.2, it could be proved that
$
\| \thetakmone - \big\{\ddot{\mL}(\thetaglo) \big\}^{-1}\dot{\mL} (\thetakmone) - \thetaglo \| \leq \kappa_{k-1}^{\prime} \|\thetakmone - \thetaglo \|^2
$
with probability tending to 1.
Because the first term in (\ref{apc:k-step}) is a negligible higher-order term, it suffices to study the upper bound of the second term in (\ref{apc:k-step}).

Denote $\Delta_3 = \big[\big\{ \ddot{\mL} (\thetaglo)\big\}^{-1} - B_{k-1}^{-1} \big] \dot{\mL}(\thetakmone)$; then it could be verified that
\beqrs
\| \Delta_3\| &=&\big\| \big\{ \ddot{\mL}(\thetaglo) \big\}^{-1} \big\{ B_{k-1}^{-1} - \ddot{\mL}(\thetaglo) \big\} B_{k-1}^{-1} \dot{\mL}(\thetakmone) \big\| \\
&\leq& \big\| \big\{ \ddot{\mL}(\thetaglo) \big\}^{-1} \big\|_2 \big\| B_{k-1}^{-1} - \ddot{\mL}(\thetaglo) \big\|_2 \big\| B_{k-1}^{-1} \big\|_2 \|\ddot{\mL}(\xi_{k-1}) \big\|_2 \| \thetakmone - \thetaglo\|.
\eeqrs
By Lemmas \ref{lm:new-event} and \ref{lm:linearly-converge}, we know $\big\| \big\{ \ddot{\mL}(\thetaglo) \big\}^{-1} \big\|_2 $ and $\big\|\ddot{\mL}(\xi_{k-1}) \big\|_2$ are both bounded by some constant $C>0$ with probability tending to 1. As a consequence, to prove Corollary \ref{thm:K-step-updating}, it suffices to prove that for any $2 \leq k \leq K$,
$\| B_k - \ddot{\mL} (\thetaglo) \| \leq \kappa_k \Big(  M^{-1} \sum_{m=1}^M \Big[ \|\thetam - \theta_0\|^2 +  \| \hesmtrue - \Omega(\theta_0) \|_2^2 + \big\|\big\{ \dddot{\mL}_{(m)}(\theta_0) - \dot{\Omega}(\theta_0) \big\} \big\{ (\thetam - \theta_0) \otimes I_p \big\} \big\|_2 \Big] + \| \thetap - \theta_0 \|\Big)$ with probability tending to 1. We then verify the inequality by the inductive method under the SR1 and BFGS update separately as follows.

{\sc Part 1 (SR1).}
Assume that
$
\| B_{k-1} - \ddot{\mL} (\thetaglo) \|_2 \leq \kappa_{k-1} \Big(  M^{-1} \sum_{m=1}^M \Big[ \|\thetam - \theta_0\|^2 +  \| \hesmtrue - \Omega(\theta_0) \|_2^2 + \big\| \big\{ \dddot{\mL}_{(m)}(\theta_0) - \dot{\Omega}(\theta_0) \big\} \big\{ (\thetam - \theta_0) \otimes I_p \big\} \big\|_2 \Big] + \| \thetap - \theta_0 \|\Big).
$
The goal is to verify that
$\| B_{k} - \ddot{\mL} (\thetaglo) \|_2 \leq \kappa_{k} \Big(  M^{-1} \sum_{m=1}^M \Big[ \|\thetam - \theta_0\|^2 +  \| \hesmtrue - \Omega(\theta_0) \|_2^2 + \big\| \big\{ \dddot{\mL}_{(m)}(\theta_0) - \dot{\Omega}(\theta_0) \big\} \big\{ (\thetam - \theta_0) \otimes I_p \big\} \big\|_2 \Big] + \| \thetap - \theta_0 \|\Big)$.
To this end, by the SR1 updating formula and (\ref{apd:well-define}), we have
$$
\| B_{k} - \ddot{\mL}(\thetaglo) \|_2 \leq \| B_{k-1} - \ddot{\mL}(\thetaglo) \|_2 + \frac{\|r_{k-1}\|}{c_1 \|s_{k-1}\|}.
$$
Furthermore, it could be proved that, with probability tending to 1, we have
\beqrs
\frac{\|r_{k-1}\|}{c_1 \|s_{k-1}\|} &\leq& \frac{ \| \ddot{\mL}(\thetakmone) - B_{k-1} \|_2 \|s_{k-1}\| + \big\| \ddot{\mL}(\xi_{k-1}) - \ddot{\mL}(\thetakmone) \big\|_2 \|s_{k-1} \|}{c_1 \|\thetakm - \thetakmone \| }
\\
&\leq& c_1^{-1} \Big\{ \| \ddot{\mL}(\thetaglo) - B_{k-1} \|_2 + \| \ddot{\mL}(\thetakmone) - \ddot{\mL}(\thetaglo) \|_2 + \big\| \ddot{\mL}(\xi_{k-1}) - \ddot{\mL}(\thetakmone) \big\|_2 \Big\} \\
&\leq & \kappa_{k} \Big(  M^{-1} \sum_{m=1}^M \Big[ \|\thetam - \theta_0\|^2 +  \| \hesmtrue - \Omega(\theta_0) \|_2^2 + \| \thetap - \theta_0 \|\\
&& + \big\{ \dddot{\mL}_{(m)}(\theta_0) - \dot{\Omega}(\theta_0) \big\} \big\{ (\thetam - \theta_0) \otimes I_p \big\} \Big]\Big).
\eeqrs
The last inequality holds because $\|\thetakmone - \thetaglo\|$ is a higher-order term compared to $M^{-1} \sum_{m=1}^M \Big[ \|\thetam - \theta_0\|^2 +  \| \hesmtrue - \Omega(\theta_0) \|_2^2 +\big\| \big\{ \dddot{\mL}_{(m)}(\theta_0) - \dot{\Omega}(\theta_0) \big\} \big\{ (\thetam - \theta_0) \otimes I_p \big\} \big\|_2 \Big] + \| \thetap - \theta_0 \|$. This finishes the proof of the first part.

{\sc Part 2 (BFGS).} Similar to the proof of {\sc Part 1}, assume that
$
\| B_{k-1} - \ddot{\mL} (\thetaglo) \|_2 \leq \kappa_{k-1} \Big(  M^{-1} \sum_{m=1}^M \Big[ \|\thetam - \theta_0\|^2 +  \| \hesmtrue - \Omega(\theta_0) \|_2^2 + \big\| \big\{ \dddot{\mL}_{(m)}(\theta_0) - \dot{\Omega}(\theta_0) \big\} \big\{ (\thetam - \theta_0) \otimes I_p \big\} \big\|_2 \Big] + \| \thetap - \theta_0 \|\Big)
$. By the BFGS updating formula and \citet[Lemma 5.1]{broyden1973local}, we have
\beqrs
E_k = P_{k-1}^\top E_{k-1} P_{k-1} + \frac{ \big\{ \ddot{\mL}(\thetaglo) \big\}^{-1/2} \big\{ y_{k-1} - \ddot{\mL}(\thetaglo) s_{k-1} \big\} \Big[ \ddot{\mL}(\thetaglo)^{-1/2} y_{k-1}\Big]^\top }{y_{k-1}^\top s_{k-1}} \\
+ \frac{\big\{ \ddot{\mL}(\thetaglo) \big\}^{-1/2} y_{k-1} \big\{y_{k-1} - \ddot{\mL}(\thetaglo) s_{k-1}\big\}^\top \big\{ \ddot{\mL}(\thetaglo) \big\}^{-1/2} P_{k-1}}{ y_{k-1}^\top s_{k-1}},
\eeqrs
where
$$
P_{k-1} = I_p - \frac{\big\{\ddot{\mL}(\thetaglo)\big\}^{1/2}s_{k-1} \Big[ \big\{\ddot{\mL}(\thetaglo)\big\}^{-1/2} y_{k-1} \Big]^\top  }{y_{k-1}^\top s_{k-1}},
$$
$E_k = \big\{ \ddot{\mL}(\thetaglo) \big\}^{-1/2} \big\{ B_k - \ddot{\mL}(\thetaglo) \big\} \big\{ \ddot{\mL}(\thetaglo) \big\}^{-1/2}$ and $E_{k-1} = \big\{ \ddot{\mL}(\thetaglo) \big\}^{-1/2} \big\{ B_{k-1} - \ddot{\mL}(\thetaglo) \big\} $ $\big\{ \ddot{\mL}(\thetaglo) \big\}^{-1/2}$.
Then, using similar analytical techniques as in {\sc Part 2} of Appendix A.2, and that $|y_{k-1}^\top s_{k-1}| \geq c_1 \|s_{k-1}\|^2$ for some positive constant $c_1 > 0$, it could be verified that
$
\| E_k\| \leq C \big\{\| B_{k-1} - \ddot{\mL} (\thetaglo) \|_2 + \| \thetakmone - \thetaglo \|\big\}
$
with probability tending to 1. This finishes the proof of the second part.

Next, applying (\ref{apa:one-stage-order}) and the inductive method again, it could be proved that
$
\|\thetakm - \thetaglo\| = O_p\big( p^{k} (\log p)^{(k+1)/2}/n^{k+1} \big) + o_p\big(1/\sqrt{N}\big).
$
As a consequence, under the condition $N\big\{ p^{2k} (\log p)^{k+1}\big\}/n^{2k+2} $ $\to 0$, we have $\|\thetakm - \thetaglo\| = o_p(N^{-1/2})$, which accomplishes the whole corollary proof.

\renewcommand{\theequation}{B.\arabic{equation}}
\setcounter{equation}{0}
\noindent
\scsection{Appendix B: Some Useful Lemmas}

\begin{lemma}
\label{lm:est-expectation}
Assume the technical conditions (C1)--(C6) hold. Then, the following equations hold for some positive constants $C_1$--$C_4$ and $1 \leq k \leq 4$.
\begin{gather}
E\big\{ \|\thetam - \theta_0 \|^k \big\} \leq C_1 n^{-k/2} C_G^2\big\{ 1 + o(1) \big\} \\
E \big\{  \| \hesmtrue - \Omega(\theta_0)  \|^k_2 \big\} \leq C_2 \frac{\log^{k/2} (2p) C_H^k}{n^{k/2}} \\
E\big[ \| \thetap - \theta_0 \|^2  \big] \leq \bigg\{ \frac{2 C_G^2}{ \tau_{\min}^2 N } + \frac{C_3 C_G^2 }{ \tau_{\min}^4 n^2 } \Big( C_H^2 \log p + \frac{C_{\max}^2 C_G^2}{\tau_{\min}^2} \Big) \bigg\} \big\{ 1 + o(1)\big\} \\
E \Big\| \big\{ \dddot{\mL}_{(m)}(\theta_0) - \dot{\Omega}(\theta_0) \big\} \big\{ (\thetam - \theta_0) \otimes I_p \big\}\Big\|_2 \leq  C_4 \frac{ p \sqrt{\log p}}{n} \big\{ 1 + o(1) \big\}.
\label{lmeq:3rd-order}
\end{gather}
\end{lemma}
\begin{proof}
Given (B.1)--(B.3) in Theorem 1, B.0.1, Lemma 7 in \cite{zhang2013communication} and (C6), it suffices to verify (\ref{lmeq:3rd-order}). To this end, operator $G_{(m)} = \dot\mL_{(m)} - \dot{E}(\mL_{(m)})$, and $G_{(m)}^j$ represents the $j$-th element of $G_{(m)}$; then we have
$\big\{ \dddot{\mL}_{(m)}(\theta_0) - \dot{\Omega}(\theta_0) \big\} \big\{ (\thetam - \theta_0) \otimes I_p \big\} = \Big[ \ddot{G}_{(m)}^1(\theta_0) (\thetam - \theta_0) ,\dots, \ddot{G}_{(m)}^p(\theta_0) (\thetam - \theta_0) \Big]$. Consequently, by the Cauchy--Schwarz inequality, it could be proved that
\beqrs
&&E\big\| \big\{ \dddot{\mL}_{(m)}(\theta_0) - \dot{\Omega}(\theta_0) \big\} \big\{ (\thetam - \theta_0) \otimes I_p \big\} \big\|_2 \leq \sum_{j=1}^p E\big\| \ddot{G}_{(m)}^j(\theta_0) (\thetam - \theta_0) \big\| \\
&\leq& \sum_{j=1}^P \Big\{ E \| \ddot{G}_{(m)}^j(\theta_0) \|^2_2 E\|\thetam - \theta_0\|^2\Big\}^{1/2}.
\eeqrs
The first inequality holds because $\|B\|_2 \leq \|B\|_F$ for any matrix $B$, and $\|\cdot\|_F$ represents the Frobenius norm. By Lemma 16 in \cite{zhang2013communication}, we have
$
E \| \ddot{G}_{(m)}^j(\theta_0) \|^2_2 \leq O(\log p/n).
$
This leads to $E\big\| \big\{ \dddot{\mL}_{(m)}(\theta_0) - \dot{\Omega}(\theta_0) \big\} \big\{ (\thetam - \theta_0) \otimes I_p \big\} \big\|_2 \leq O\big( p \sqrt{\log p}/n\big)$.
\end{proof}

\begin{lemma}
\label{lm:bad-event}
Assume the technical conditions (C1)--(C6) hold. Then, we have
$
P\Big( \bigcup_{m=0}^M \mE_{m}^c \Big) $ $ \to 0,
$
where $\mE_{m}$s are defined in (\ref{apb:good-event}).
\end{lemma}
\begin{proof}
The proof is shown in Lemma 7 in \cite{zhang2013communication}, and D.1 in \cite{jordan2019communication}.
\end{proof}

\begin{lemma}
\label{lm:new-event}
Assume the technical conditions (C1)--(C6) hold. Let new events $\mE^{\prime}_{(m)}=$
\beqrs
\Big\{ \lambda_{\min} \big\{ \hestheta \big\} &\geq& \frac{(1-\delta) \tau_{\min}}{2} \text{ for } \theta\in\{\theta_0,\thetap,\thetaglo\},
\| \thetap - \thetaglo \| \leq \frac{\tau_{\min}}{2 C_{\max}} := \delta^{\prime} ,\\
\lambda_{\min} \big\{ \hesmtheta \big\} &\geq& \frac{(1-\delta) \tau_{\min}}{2} \text{ for } \theta\in\{\theta_0,\thetam\}, \\
\max_{\theta \in B(\thetaglo, \delta^{\prime})} \| \hestheta\|_2 &\leq& 2 C_{\max} \delta^{\prime} + \frac{\delta \tau_{\min}}{4} + \tau_{\max} := C_{\max}^{\prime} \Big\}.
\eeqrs
We have $P\big( \bigcap_{m=0}^M \mE_m \big) < P(\bigcap_{m=0}^M\mE^{\prime}_{m})$, where $\mE_{m}$s are defined in (\ref{apb:good-event}).
\end{lemma}
\begin{proof}
We analyze the three terms under the event $\bigcap_{m=0}^M \mE_m$ separately. First, we prove that
\beqrs
\lambda_{\min} \big\{ \hesint \big\} &\geq& \lambda_{\min} \big\{\Omega(\theta_0)\big\} - \| \hestrue - \Omega(\theta_0)\|_2 - \| \hesint - \hestrue \|_2 \\
&\geq& \tau_{\min} - \frac{\delta \tau_{\min}}{2} - 2 C_{\max} \| \thetap - \theta_0 \|_2 \geq \frac{(1-\delta) \tau_{\min}}{2}.
\eeqrs
The first inequality holds because
$
\lambda_{\min}(B_1 ) = \min_{\|u\|=1} u^\top (B_1 - B_2 + B_2) u \geq \min_{u_1=1} u_1^\top $  $(B_1 - B_2) u_1 + \min_{u_2=1} u_2^\top B_2 u_2 \geq - \|B_1-B_2\|_2 + \lambda_{\min}(B_2)
$ for any symmetric matrixes $B_1$ and $B_2$.
The last inequality holds because $\|\thetap - \theta_0 \| \leq \tau_{\min} / (4 C_{\max})$ under $\bigcap_{m=0}^M \mE_m$. By similar technical analysis, we know that $\lambda_{\min} \big\{ \hestheta \big\} \geq (1-\delta) \tau_{\min}/2$ when $\theta= \theta_0$ or $\theta = \thetaglo$ and $\lambda_{\min} \big\{ \hesmtheta \big\} \geq (1-\delta) \tau_{\min}/2 \text{ for } \theta\in\{\theta_0,\thetam\}$.
Next, it is obvious that $\| \thetap - \thetaglo\| \leq \| \thetap - \theta_0 \| + \| \thetaglo - \theta_0 \| \leq  \tau_{\min} / (2 C_{\max}) $ under $\bigcap_{m=0}^M \mE_m$. Moreover, using the triangle inequality, we have
\beqrs
\max_{\theta \in B(\thetaglo, \delta^{\prime})} \| \hestheta\|_2 &\leq& \max_{\theta \in B(\thetaglo, \delta^{\prime})} \| \hestheta - \hestrue \|_2 + \| \hestrue - \Omega(\theta_0) \|_2 + \| \Omega(\theta_0) \|_2 \\
&\leq& 2 C_{\max} \delta^{\prime} + \frac{\delta \tau_{\min}}{4} + \tau_{\max}.
\eeqrs
This accomplishes the whole lemma proof.
\end{proof}

\begin{lemma}
\label{lm:linearly-converge}
Assume the technical conditions (C1)--(C6) hold. Then, there exists a positive constant $\rho_K < 1$ such that for any $K > 0$,
$
\| \thetaK - \thetaglo \| \leq \rho_K^{K-1} \| \thetaos - \thetaglo \|
$
with probability tending to 1.
\end{lemma}

\begin{proof}
We first verify the lemma when $K=2$. For the SR1 update,
\beqrs
\thetatwo - \thetaglo & =& \thetaos - H_1 \dot{\mL}(\thetaos) - \thetaglo \\
&=& \thetaos - \thetaglo - H_1 \big\{ \dot{\mL}(\thetaos) - \dot{\mL}(\thetaglo) \big\} \\
&=& \thetaos - \thetaglo - H_1 \ddot{\mL}(\thetaglo)(\thetaos - \thetaglo) + H_1 \Big[ \ddot{\mL}(\thetaglo) \big( \thetaos - \thetaglo \big) \\
&& - \big\{ \dot{\mL}(\thetaos) - \dot{\mL}(\thetaglo) \big\} \Big]\\
&=& H_1 \big\{ B_1 - \ddot{\mL}(\thetaglo) \big\} \big( \thetaos - \thetaglo \big) + \mO.
\eeqrs
Here $\mO = H_1 \big[ \ddot{\mL}(\thetaglo) \big( \thetaos - \thetaglo \big) - \big\{ \dot{\mL}(\thetaos) - \dot{\mL}(\thetaglo) \big\} \big]$. By Taylor's expansion, $\mO$ is a negligible higher-order term. Then, it suffices to analyze $H_1 \big\{ B_1 - \ddot{\mL}(\thetaglo) \big\}$. We have
$$
B_1 - \ddot{\mL}(\thetaglo) = B_0 - \ddot{\mL}(\thetaglo) + \frac{ (y_0 - B_0 s_0 )(y_0 - B_0 s_0)^\top }{ (y_0 - B_0 s_0)^\top y_0} .
$$
By similar analysis to that in Appendix A Sections A.1 and A.2, it could be easily found that $B_0 - \ddot{\mL}(\thetaglo)$ and $\{ (y_0 - B_0 s_0)^\top y_0\}^{-1} (y_0 - B_0 s_0 )(y_0 - B_0 s_0)^\top$ both converge to 0 in probability.
Consequently, with probability tending to 1, there exists a small positive number $\rho_1 < 1$, such that
$
\| B_1 - \ddot{\mL}(\thetaglo) \|_2 \|H_1 \|_2 \leq \rho_1/2.
$
Thus, we prove that
$
\| \thetatwo - \thetaglo \| \leq \rho_1 \| \thetaos - \thetaglo \| .
$

Similarly, we obtain the linear convergence rate for $\thetaK$ using the same analytical techniques used to study $\thetatwo$. This finishes the lemma for SR1 updating. The proof for BFGS updating is shown in Lemma 3 in \cite{mokhtari2018iqn}.
\end{proof}

\newpage
\renewcommand{\theequation}{C.\arabic{equation}}
\setcounter{equation}{0}
\noindent
\scsection{Appendix C: Additional Numerical Details}

\scsubsection{C.1: Distributed $K$-Stage SR1 Algorithm}
\noindent

We introduce the detailed multi-stage algorithm with the SR1 updating strategy, not specified in Section 2.5. The key idea of the SR1 updating strategy is the same as that of BFGS updating strategy; that is, we establish a distributed version from the classical single computer updating formula. Nevertheless, the denominator in (\ref{eq:rankone}) involves $H^{(t)}$. As a result, we need to design a distributed algorithm more skillfully, so that the updated matrix of the distributed version is equivalent to that of the global one, and the number of communication rounds remains the same as that of the distributed BFGS updating algorithm. We first define $v_{k}  = \wh\theta_{\operatorname{stage},k+1} - \thetakm - H_{k-1} \big\{\dot{\mL}(\wh\theta_{\operatorname{stage},k+1}) - \dot{\mL}(\thetakm) \big\}$ for any $k \geq 1$. In particular, we denote $H_{(m,-1)} = H_{(m,0)}$ and $v_{-1} = y_{-1} = \mathbf{0}$. The specific algorithm is given in Algorithm \ref{alg:SR1-two-step}.

\begin{algorithm}
\caption{Distributed $K$-Stage Quasi-Newton (SR1) Algorithm}\label{alg:SR1-two-step}
\KwIn{DQN($K\!-\!1$) estimator $\thetakone$, $v_{K-3}$ on the central computer, $\thetaktwo$, $ \dot{\mL} (\thetaktwo), y_{K-3}$, and Hessian inverse approximation $H_{(m,K-3)}$ on the $m$-th worker\; }
\KwOut{DQN(K) estimator $\thetaK$\;}
The central computer broadcasts $\thetakone$ and $v_{K-3}$ to each worker\;
\For{$m=1,2,\dots,M$ (distributedly)}{
Compute $\dot{\mL}_{(m)}(\thetakone) $ and
transfer it to the central computer\;
Update local Hessian inverse approximation by
$
H_{(m,K-2)} = H_{(m,K-3)} + \big[v_{K-3}^\top y_{K-3} \big]^{-1} v_{K-3} v_{K-3}^\top
$\;
}
The central computer computes $\dot{\mL}(\thetakone)= M^{-1} \sum_{m=1}^M \dot{\mL}_{(m)}(\thetakone)$ and broadcasts it to each worker \;
\For{$m=1,2,\dots,M$ (distributedly)}{
Compute $v_{(m,K-2)} = s_{K-2} - H_{(m,K-2)} y_{K-2} $ and transfer it to the central computer \;
Calculate $H_{(m,K-2)} \dot{\mL} ( \thetakone)$ and
transfer it to the central computer.
}
The central computer computes $v_{K-2} = M^{-1} \sum_{m=1}^M v_{(m,K-2)}$ and $p_{K-1} = M^{-1} \sum_{m=1}^M H_{(m,K-2)} \dot{\mL} ( \thetakone) + v_{K-2} v_{K-2}^\top \dot{\mL} ( \thetakone) / (v_{K-2}^\top y_{K-2} ) $, and obtains
$
\thetaK = \thetakone - p_{K-1}.
$
\end{algorithm}

\scsubsection{C.2: Distributed $K$-Stage Newton--Raphson Algorithm}
\noindent

Next, we introduce the detailed multi-stage algorithm but using the Newton--Raphson updating strategy, which is not specified in Section 3.2.
\begin{algorithm}
\caption{Distributed $K$-Stage Newton Algorithm}\label{alg:newton-step}
\KwIn{$K\!-\!1$-stage estimator $\thetakone$ on the central computer\; }
\KwOut{K-stage estimator $\thetaK$\;}
The central computer broadcasts $\thetakone$ to each worker\;
\For{$m=1,2,\dots,M$ (distributedly)}{
Compute $\dot{\mL}_{(m)}(\thetakone) $ and
transfer it to the central computer\;
}
The central computer computes $\dot{\mL}(\thetakone)= M^{-1} \sum_{m=1}^M \dot{\mL}_{(m)}(\thetakone)$ and broadcasts it to each worker \;
\For{$m=1,2,\dots,M$ (distributedly)}{
Compute $\big\{\ddot{\mL}_{(m)}(\thetakone)\big\}^{-1} \dot{\mL}(\thetakone)$ and transfer it to the central computer\;
}
The central computer computes
$
\thetaK = \thetakone - M^{-1} \sum_{m=1}^M \big\{\ddot{\mL}_{(m)}(\thetakone)\big\}^{-1} \dot{\mL}(\thetakone).
$
\end{algorithm}

\renewcommand{\theequation}{D.\arabic{equation}}
\setcounter{equation}{0}
\noindent
\scsection{Appendix D: Updating Method of Quasi-Newton Matrix}
\noindent

We introduce the detailed intuition and proofs to derive (\ref{eq:rankone}) and (\ref{eq:BFGS}) in the main text. To this end, denote $y^{(t)} = \dot{\mL}(\wh \theta^{(t+1)}) - \dot{\mL}(\wh \theta^{(t)}) $ and $s^{(t)} = \wh \theta^{(t+1)} - \wh \theta^{(t)}$.

{\sc 1. SR1 Updates.} Equation (\ref{eq:rankone}) is the simplest quasi-Newton matrix updating formula.
Let $\HNT$ be the $t$-th approximated Hessian inverse; we then derive $\HNTT$, satisfying the secant condition in (\ref{eq:secant-condi}), using rank one updating. To this end,
we use the undetermined coefficient method, assuming that
\beq
\label{apd:sr1-updating}
\HNTT = \HNT + \alpha u u^\top
\eeq
for some undetermined coefficient $u \in \mR^p$ and $\alpha \in \mR$. Then, according to (\ref{eq:secant-condi}), we have
$s^{(t)} = H^{(t+1)} y^{(t)} = \big( H^{(t)} + \alpha u u^\top \big) y^{(t)} $. Then, it could be proved that
\beq
\label{apd:sr1-detail}
\alpha u^\top y^{(t)} u = s^{(t)} - H^{(t)} y^{(t)} .
\eeq
Note that $\alpha u^\top y^{(t)} \in \mR $ is a scale, indicating that $u$ and $s^{(t)} - H^{(t)} y^{(t)}$ share the same direction. Hence, we denote $u = s^{(t)} - H^{(t)} y^{(t)};$ then (\ref{apd:sr1-detail}) could be rewritten as
$
\alpha \big( s^{(t)} - H^{(t)} y^{(t)} \big)^\top y^{(t)} \big( s^{(t)} - H^{(t)} y^{(t)}\big) = s^{(t)} - H^{(t)} y^{(t)}.
$
Thus, we have $\alpha = \Big\{ \big( s^{(t)} - H^{(t)} y^{(t)} \big)^\top y^{(t)} \Big\}^{-1}$.
Applying the results back to (\ref{apd:sr1-updating}) leads to
\beq
\label{apd:sr1-H}
\HNTT = \HNT + \frac{ \big( s^{(t)} - H^{(t)} y^{(t)} \big)\big( s^{(t)} - H^{(t)} y^{(t)} \big)^\top }{\big( s^{(t)} - H^{(t)} y^{(t)} \big)^\top y^{(t)}}.
\eeq
According to the Sherman--Morrison equation \cite[Theorem 10.8]{burden2015numerical}, (\ref{apd:sr1-H}) could be rewritten as
\beq
\label{apd:sr1-B}
B^{(t+1)} = B^{(t)} + \frac{ \big( y^{(t)} - B^{(t)} s^{(t)} \big)\big( y^{(t)} - B^{(t)} s^{(t)} \big)^\top }{\big( y^{(t)} - B^{(t)} s^{(t)} \big)^\top s^{(t)}}.
\eeq
Here $B^{(t)} = \big\{ H^{(t)} \big\}^{-1}$.
When using the SR1 updating formula, it should be well defined. Consequently, (\ref{apd:sr1-H}) and (\ref{apd:sr1-B}) would be used only if
\begin{gather}
\Big| \big(s^{(t)} - H^{(t)} y^{(t)}\big)^\top s^{(t)} \Big| \geq c_1 \|s^{(t)} - H^{(t)} y^{(t)} \| \|s^{(t)} \| \text{ or } \nonumber\\
\Big|\big( y^{(t)} - B^{(t)} s^{(t)} \big)^\top s^{(t)}\Big| \geq c_1\| y^{(t)} - B^{(t)} s^{(t)}\| \|s^{(t)}\|
\label{apd:well-define}
\end{gather}
for some positive constant $0 < c_1 < 1$. Otherwise, we keep $H^{(t+1)} = H^{(t)}$ or $B^{(t+1)} = B^{(t)};$ see more discussions in \cite{conn1991convergence} and \cite{nocedal1999numerical}.

{\sc 2. SR2 (BFGS) Updates.}
SR1 updating is simple and easy to conduct. However, the positive definiteness of the approximated matrix (i.e., $H^{(t)}$) cannot be guaranteed; that is, we cannot ensure $\big(y^{(t)} - B^{(t)} s^{(t)}\big)^\top s^{(t)} > 0$. The SR2 updating formula was proposed to address this problem.
Similar to SR1 updating, given $H^{(t)}$, we consider the determined coefficient method to obtain the updating matrix $H^{(t+1)}$. Here, for simplicity, instead of directly analyzing $H^{(t)}$, we consider $B^{(t)} = \big\{ H^{(t)} \big\}^{-1}$ first. Thus, given $B^{(t)}$, assume that
\beq
\label{apd:BFGS-updating}
B^{(t+1)} = B^{(t)} + a u u^\top + b v v^\top, \nonumber
\eeq
where $u,v \in \mR^p$ and $a,b \in \mR$. When
\beq
\label{eq:secant-condi-eq}
\big\{ \dot{\mL} (\wh \theta^{(t+1)}) - \dot{\mL} (\wh \theta^{(t)}) \big\} =B^{(t+1)}  ( \wh \theta^{(t+1)} - \wh \theta^{(t)}),
\eeq
we obtain
$
\big( B^{(t)} + a u u^\top + b v v^\top \big) s^{(t)} = y^{(t)}
$;
this leads to
$$
\big( a u^\top s^{(t)} \big) u + \big( b v^\top s^{(t)} \big) v = y^{(t)} - B^{(t)} s^{(t)}.
$$
Of the many ways to determine $u$ and $v$, we consider the following criterion:
$
u = y^{(t)}, a u^\top s^{(t)} = 1, v = B^{(t)} s^{(t)}$, and $b v^\top s^{(t)} = -1.
$
Consequently, (\ref{apd:BFGS-updating}) could be rewritten as
\beq
\label{apd:bfgs-B}
B^{(t+1)} = B^{(t)} + \frac{y^{(t)} \big\{y^{(t)}\big\}^\top}{ \big\{s^{(t)}\big\}^\top y^{(t)} } - \frac{B^{(t)} s^{(t)} \big( B^{(t)} s^{(t)}  \big)^\top }{ \big\{s^{(t)} \big\}^\top B^{(t)} s^{(t)} }. \nonumber
\eeq
Finally, according to the Sherman--Morrison formula \cite[Theorem 10.8]{burden2015numerical}, we obtain the updating formula (\ref{eq:BFGS}).

Moreover, there is another method to derive (\ref{eq:BFGS}). To be more precise, $H^{(t+1)}$ is exactly the solution of the following optimal problem:
\begin{gather}
\min_H \| H - H^{(t)} \|_W, \nonumber\\
s.t. H =H^\top, Hy^{(t)} = s^{(t)}.
\label{apd:BFGS-2}
\end{gather}
Here $\|H\|_W = \|W^{1/2} H W^{1/2} \|_F$ represents the weighted Frobenius norm, and $W$ could be any matrix that satisfies $W s^{(t)} = y^{(t)}$.
Analyzing the optimal problem using (\ref{apd:BFGS-2}), we find that the solution of the problem (i.e., $H^{(t+1)}$) is a matrix $H$ that is the closest to $H^{(t)}$. In addition, $H^{(t+1)}$ should be symmetric and satisfy the secant condition (\ref{eq:secant-condi-eq}). For convex loss functions, it leads to $(s^{(t)})^\top y^{(t)} > 0;$ thus, when $H^{(t)}$ is a positive definite, $H^{(t+1)}$ would also be positive definite; see more details in \cite{nocedal1999numerical}.

\renewcommand{\theequation}{E.\arabic{equation}}
\setcounter{equation}{0}
\renewcommand{\thetable}{E.\arabic{table}}
\renewcommand{\thefigure}{E.\arabic{figure}}
\noindent
\scsection{Appendix E: Supplementary Numerical Results}
\noindent

In this subsection, we provide the supplementary numerical results , which are not presented
in the Section 3.1.
Specifically, to evaluate the performance of our proposed DQN method in {\sc Example 2}, we report the log(MSE), SD, IQR, and range of log(MSE). The numerical results with different dimension $p$ and local sample size $n$
are given in Tables \ref{ap:tab:mse-p} and \ref{ap:tab:mse-n}, respectively. All the results are qualitatively similar to those in the main text.

\begin{table}
\centering
\caption{ Log(MSE) values and corresponding SD, IQR and range for {\sc Examples 2}. The numerical performance is evaluated for different methods with different feature dimensions $p (\times 10^2)$.  The whole sample size $N$ and local sample size $n$ are fixed to be $10^6$ and $2 \times 10^4$, respectively. The reported results are averaged for $R=100$ simulation replications. }\label{ap:tab:mse-p}
\begin{spacing}{1}
\setlength{\tabcolsep}{1.1mm}{
\begin{tabular}{cc|cc|cc|cc|cc|cc|c}
\hline
\hline
& & \multicolumn{2}{c|}{Stage 0} & \multicolumn{2}{c|}{Stage 1} & \multicolumn{2}{c|}{Stage 2} & \multicolumn{2}{c|}{Stage 3} & \multicolumn{2}{c|}{Stage 4} &\\
&$p$ &SR1 &BFGS &SR1 &BFGS&SR1 &BFGS&SR1 &BFGS&SR1 &BFGS&MLE \\[0.1em]
\hline
log&1&-9.19&-9.19&-9.19&-9.19&-9.19&-9.19&-9.19&-9.19&-9.19&-9.19&-9.19\\
(MSE)&10&-6.84&-6.84&-6.91&-6.91&-6.91&-6.91&-6.91&-6.91&-6.91&-6.91&-6.91\\
&20&-6.09&-6.09&-6.20&-6.20&-6.21&-6.21&-6.22&-6.22&-6.22&-6.22&-6.22\\
\multicolumn{13}{c}{} \\
SD&1&0.14&0.14&0.14&0.14&0.14&0.14&0.14&0.14&0.14&0.14&0.14\\
&10&0.05&0.05&0.05&0.05&0.05&0.05&0.05&0.05&0.05&0.05&0.05\\
&20&0.03&0.03&0.03&0.03&0.03&0.03&0.03&0.03&0.03&0.03&0.03\\
\multicolumn{13}{c}{} \\
IQR&1&0.18&0.18&0.19&0.19&0.19&0.19&0.19&0.19&0.19&0.19&0.19\\
&10&0.07&0.07&0.07&0.07&0.07&0.07&0.07&0.07&0.07&0.07&0.07\\
&20&0.05&0.05&0.03&0.03&0.04&0.04&0.04&0.04&0.04&0.04&0.04\\
\multicolumn{13}{c}{} \\
range&1&0.66&0.66&0.65&0.65&0.65&0.65&0.65&0.65&0.65&0.65&0.65\\
&10&0.30&0.30&0.26&0.26&0.27&0.27&0.26&0.26&0.27&0.27&0.27\\
&20&0.18&0.18&0.18&0.18&0.19&0.19&0.19&0.19&0.19&0.19&0.19\\
\hline
\end{tabular}}
\end{spacing}
\end{table}

\begin{table}[h]
\caption{Log(MSE) values and corresponding SD, IQR and range for {\sc Examples 2}. The numerical performance is evaluated for different $n (\times 10^2)$ and methods.  The whole sample size $N$ and feature dimension $p$ are fixed to $N = 10^6,$ and $p = 10 ^3$, respectively. Finally, the reported results are averaged based on $R=100$ simulations.}\label{ap:tab:mse-n}
\centering
\begin{spacing}{1}
\setlength{\tabcolsep}{1.1mm}{
\begin{tabular}{cc|cc|cc|cc|cc|cc|c}
\hline
\hline
& & \multicolumn{2}{c|}{Stage 0} & \multicolumn{2}{c|}{Stage 1} & \multicolumn{2}{c|}{Stage 2} & \multicolumn{2}{c|}{Stage 3} & \multicolumn{2}{c|}{Stage 4} &\\
&$n$ &SR1 &BFGS &SR1 &BFGS&SR1 &BFGS&SR1 &BFGS&SR1 &BFGS&MLE \\[0.1em]
\hline
\multicolumn{13}{c}{} \\
log&50&-7.30&-7.30&-7.58&-7.58&-7.60&-7.60&-7.61&-7.61&-7.61&-7.61&-7.61\\
(MSE)&100&-7.49&-7.49&-7.60&-7.60&-7.60&-7.60&-7.61&-7.61&-7.61&-7.61&-7.61\\
&500&-7.60&-7.60&-7.60&-7.60&-7.61&-7.61&-7.61&-7.61&-7.61&-7.61&-7.61\\
\multicolumn{13}{c}{} \\
SD&50&0.06&0.06&0.06&0.06&0.06&0.06&0.06&0.06&0.06&0.06&0.06\\
&100&0.06&0.06&0.06&0.06&0.06&0.06&0.06&0.06&0.06&0.06&0.06\\
&500&0.06&0.06&0.06&0.06&0.06&0.06&0.06&0.06&0.06&0.06&0.06\\
\multicolumn{13}{c}{} \\
IQR&50&0.07&0.07&0.07&0.07&0.07&0.07&0.08&0.08&0.08&0.08&0.08\\
&100&0.07&0.07&0.07&0.07&0.07&0.07&0.08&0.08&0.08&0.08&0.08\\
&500&0.08&0.08&0.07&0.07&0.07&0.07&0.08&0.08&0.08&0.08&0.08\\
\multicolumn{13}{c}{} \\
range&50&0.39&0.39&0.36&0.36&0.37&0.37&0.37&0.37&0.37&0.37&0.37\\
&100&0.42&0.42&0.38&0.38&0.37&0.37&0.37&0.37&0.37&0.37&0.37\\
&500&0.39&0.39&0.36&0.36&0.36&0.36&0.37&0.37&0.37&0.37&0.37\\

\hline
\end{tabular}}
\end{spacing}
\end{table}

\renewcommand \refname{\centerline{REFERENCES}}
\bibliographystyle{asa}
\bibliography{ref}

\begin{thebibliography}{41}
\newcommand{\enquote}[1]{``#1''}
\expandafter\ifx\csname natexlab\endcsname\relax\def\natexlab#1{#1}\fi

\bibitem[{Broyden et~al.(1973)Broyden, Dennis~Jr, and
  Mor{\'e}}]{broyden1973local}
Broyden, C.~G., Dennis~Jr, J.~E., and Mor{\'e}, J.~J. (1973), \enquote{On the
  local and superlinear convergence of quasi-Newton methods,} \textit{IMA
  Journal of Applied Mathematics}, 12, 223--245.

\bibitem[{Bubeck(2015)}]{bubeck2015theory}
Bubeck, S. (2015), \enquote{Theory of convex optimization for machine
  learning,} \textit{Foundations and Trends in Machine Learning}, 8.

\bibitem[{Burden et~al.(2015)Burden, Faires, and Burden}]{burden2015numerical}
Burden, R.~L., Faires, J.~D., and Burden, A.~M. (2015), \textit{Numerical
  analysis}, Cengage learning.

\bibitem[{Chen et~al.(2014)Chen, Wang, and Zhou}]{chen2014bfgs}
Chen, W., Wang, Z., and Zhou, J. (2014), \enquote{Large-scale L-BFGS using
  MapReduce,} \textit{Advances in neural information processing systems}, 27.

\bibitem[{Conn et~al.(1991)Conn, Gould, and Toint}]{conn1991convergence}
Conn, A.~R., Gould, N.~I., and Toint, P.~L. (1991), \enquote{Convergence of
  quasi-Newton matrices generated by the symmetric rank one update,}
  \textit{Mathematical programming}, 50, 177--195.

\bibitem[{Crane and Roosta(2019)}]{crane2019dingo}
Crane, R. and Roosta, F. (2019), \enquote{DINGO: Distributed Newton-type method
  for gradient-norm optimization,} \textit{Advances in Neural Information
  Processing Systems}, 32.

\bibitem[{Davidon(1991)}]{davidon1991variable}
Davidon, W.~C. (1991), \enquote{Variable metric method for minimization,}
  \textit{SIAM Journal on Optimization}, 1, 1--17.

\bibitem[{Eisen et~al.(2017)Eisen, Mokhtari, and Ribeiro}]{Eisen2017}
Eisen, M., Mokhtari, A., and Ribeiro, A. (2017), \enquote{Decentralized
  quasi-Newton methods,} \textit{IEEE Transactions on Signal Processing}, 65,
  2613--2628.

\bibitem[{Fan and Li(2001)}]{fan2001variable}
Fan, J. and Li, R. (2001), \enquote{Variable selection via nonconcave penalized
  likelihood and its oracle properties,} \textit{Journal of the American
  statistical Association}, 96, 1348--1360.

\bibitem[{Fan and Lv(2008)}]{Fan:Lv:2008}
Fan, J. and Lv, J. (2008), \enquote{Sure independence screening for ultra-high
  dimensional feature space (with discussion),} \textit{Journal of the Royal
  Statistical Society, Series B}, 70, 849--911.

\bibitem[{Fan and Song(2010)}]{Fan:Song:2010}
Fan, J. and Song, R. (2010), \enquote{Sure independent screening in generalized
  linear models with NP-dimensionality,} \textit{Annals of Statistics}, 38,
  3567--3604.

\bibitem[{Fan et~al.(2019)Fan, Wang, Wang, and Zhu}]{fan2019distributed}
Fan, J., Wang, D., Wang, K., and Zhu, Z. (2019), \enquote{Distributed
  estimation of principal eigenspaces,} \textit{Annals of statistics}, 47,
  3009.

\bibitem[{Goldfarb(1970)}]{goldfarb1970family}
Goldfarb, D. (1970), \enquote{A family of variable-metric methods derived by
  variational means,} \textit{Mathematics of computation}, 24, 23--26.

\bibitem[{Gopal and Yang(2013)}]{gopal2013distributed}
Gopal, S. and Yang, Y. (2013), \enquote{Distributed training of large-scale
  logistic models,} in \textit{International Conference on Machine Learning},
  PMLR, pp. 289--297.

\bibitem[{Goyal et~al.(2017)Goyal, Doll{\'a}r, Girshick, Noordhuis, Wesolowski,
  Kyrola, Tulloch, Jia, and He}]{goyal2017accurate}
Goyal, P., Doll{\'a}r, P., Girshick, R., Noordhuis, P., Wesolowski, L., Kyrola,
  A., Tulloch, A., Jia, Y., and He, K. (2017), \enquote{Accurate, large
  minibatch sgd: Training imagenet in 1 hour,} \textit{arXiv preprint
  arXiv:1706.02677}.

\bibitem[{He et~al.(2013)He, Wang, and Hong}]{He:Wang:Hong:2013}
He, X., Wang, L., and Hong, H.~G. (2013), \enquote{Quantile-adaptive model-free
  variable screening for high-dimensional heterogeneous data,} \textit{Annals
  of Statistics}, 41, 342--369.

\bibitem[{Hector and Song(2020)}]{hector2020doubly}
Hector, E.~C. and Song, P. X.-K. (2020), \enquote{Doubly Distributed Supervised
  Learning and Inference with High-Dimensional Correlated Outcomes.} \textit{J.
  Mach. Learn. Res.}, 21, 173--1.

\bibitem[{Hector and Song(2021)}]{hector2021distributed}
--- (2021), \enquote{A distributed and integrated method of moments for
  high-dimensional correlated data analysis,} \textit{Journal of the American
  Statistical Association}, 116, 805--818.

\bibitem[{Huang and Huo(2019)}]{huang2019distributed}
Huang, C. and Huo, X. (2019), \enquote{A distributed one-step estimator,}
  \textit{Mathematical Programming}, 174, 41--76.

\bibitem[{Jordan et~al.(2019)Jordan, Lee, and Yang}]{jordan2019communication}
Jordan, M.~I., Lee, J.~D., and Yang, Y. (2019),
  \enquote{Communication-efficient distributed statistical inference,}
  \textit{Journal of the American Statistical Association}, 114, 668--681.

\bibitem[{Lee et~al.(2018)Lee, Lim, and Wright}]{Lee2018}
Lee, C.-p., Lim, C.~H., and Wright, S.~J. (2018), \enquote{A distributed
  quasi-Newton algorithm for empirical risk minimization with nonsmooth
  regularization,} \textit{Proceedings of the 24th ACM SIGKDD International
  Conference on Knowledge Discovery \& Data Mining}, 1646–1655.

\bibitem[{Li et~al.(2012)Li, Peng, J., and Zhu}]{Li:Peng:Zhang:2012}
Li, G., Peng, H., J., Z., and Zhu, L. (2012), \enquote{Robust rank correlation
  based screening,} \textit{Annals of Statistics}, 40, 1846--1877.

\bibitem[{Li et~al.(2020)Li, Li, Xia, and Xu}]{Li:Li:2020}
Li, X., Li, R., Xia, Z., and Xu, C. (2020), \enquote{Distributed feature
  screening via componentwise debiasing,} \textit{Journal of machine learning
  research}, 21, 1--32.

\bibitem[{Lin and Zhou(2018)}]{lin2018distributed}
Lin, S.-B. and Zhou, D.-X. (2018), \enquote{Distributed kernel-based gradient
  descent algorithms,} \textit{Constructive Approximation}, 47, 249--276.

\bibitem[{Mcdonald et~al.(2009)Mcdonald, Mohri, Silberman, Walker, and
  Mann}]{mcdonald2009efficient}
Mcdonald, R., Mohri, M., Silberman, N., Walker, D., and Mann, G.~S. (2009),
  \enquote{Efficient large-scale distributed training of conditional maximum
  entropy models,} in \textit{Advances in neural information processing
  systems}, pp. 1231--1239.

\bibitem[{Mokhtari et~al.(2018)Mokhtari, Eisen, and Ribeiro}]{mokhtari2018iqn}
Mokhtari, A., Eisen, M., and Ribeiro, A. (2018), \enquote{IQN: An incremental
  quasi-Newton method with local superlinear convergence rate,} \textit{SIAM
  Journal on Optimization}, 28, 1670--1698.

\bibitem[{Nocedal and Wright(1999)}]{nocedal1999numerical}
Nocedal, J. and Wright, S.~J. (1999), \textit{Numerical optimization},
  Springer.

\bibitem[{Qu and Li(2019)}]{qu2019accelerated}
Qu, G. and Li, N. (2019), \enquote{Accelerated distributed Nesterov gradient
  descent,} \textit{IEEE Transactions on Automatic Control}, 65, 2566--2581.

\bibitem[{Schuller(1974)}]{schuller1974order}
Schuller, G. (1974), \enquote{On the order of convergence of certain
  Quasi-Newton-methods,} \textit{Numerische Mathematik}, 23, 181--192.

\bibitem[{Shamir et~al.(2014)Shamir, Srebro, and
  Zhang}]{shamir2014communication}
Shamir, O., Srebro, N., and Zhang, T. (2014), \enquote{Communication-efficient
  distributed optimization using an approximate newton-type method,} in
  \textit{International conference on machine learning}, PMLR, pp. 1000--1008.

\bibitem[{Shao(2003)}]{shao1999mathematical}
Shao, J. (2003), \textit{Mathematical Statistics}, Springer Texts in
  Statistics. Springer.

\bibitem[{Soori et~al.(2020)Soori, Mishchenko, Mokhtari, Dehnavi, and
  Gurbuzbalaban}]{Soori2020dqn}
Soori, S., Mishchenko, K., Mokhtari, A., Dehnavi, M.~M., and Gurbuzbalaban, M.
  (2020), \enquote{DAve-QN: A distributed averaged quasi-Newton method with
  local superlinear convergence rate,} \textit{Proceedings of the Twenty Third
  International Conference on Artificial Intelligence and Statistics, PMLR},
  1965--1976.

\bibitem[{Su and Xu(2019)}]{su2019securing}
Su, L. and Xu, J. (2019), \enquote{Securing distributed gradient descent in
  high dimensional statistical learning,} \textit{Proceedings of the ACM on
  Measurement and Analysis of Computing Systems}, 3, 1--41.

\bibitem[{Tang et~al.(2020)Tang, Zhou, and Song}]{tang2020distributed}
Tang, L., Zhou, L., and Song, P. X.-K. (2020), \enquote{Distributed
  simultaneous inference in generalized linear models via confidence
  distribution,} \textit{Journal of multivariate analysis}, 176, 104567.

\bibitem[{Van~der Vaart(2000)}]{van2000asymptotic}
Van~der Vaart, A.~W. (2000), \textit{Asymptotic statistics}, vol.~3, Cambridge
  university press.

\bibitem[{Wang et~al.(2020)Wang, Huang, Zhu, and Wang}]{wangfei2020efficient}
Wang, F., Huang, D., Zhu, Y., and Wang, H. (2020), \enquote{Efficient
  Estimation for Generalized Linear Models on a Distributed System with
  Nonrandomly Distributed Data,} \textit{arXiv preprint arXiv:2004.02414}.

\bibitem[{Wang et~al.(2018)Wang, Roosta, Xu, and Mahoney}]{wang2018giant}
Wang, S., Roosta, F., Xu, P., and Mahoney, M.~W. (2018), \enquote{Giant:
  Globally improved approximate newton method for distributed optimization,}
  \textit{Advances in Neural Information Processing Systems}, 31.

\bibitem[{Zhang et~al.(2013)Zhang, Duchi, and
  Wainwright}]{zhang2013communication}
Zhang, Y., Duchi, J.~C., and Wainwright, M.~J. (2013),
  \enquote{Communication-efficient algorithms for statistical optimization,}
  \textit{The Journal of Machine Learning Research}, 14, 3321--3363.

\bibitem[{Zhang and Lin(2015)}]{zhang2015disco}
Zhang, Y. and Lin, X. (2015), \enquote{DiSCO: Distributed optimization for
  self-concordant empirical loss,} in \textit{International conference on
  machine learning}, PMLR, pp. 362--370.

\bibitem[{Zhang et~al.(2012)Zhang, Wainwright, and
  Duchi}]{zhang2012communication}
Zhang, Y., Wainwright, M.~J., and Duchi, J.~C. (2012),
  \enquote{Communication-efficient algorithms for statistical optimization,}
  \textit{Advances in neural information processing systems}, 25.

\bibitem[{Zhu et~al.(2021)Zhu, Li, and Wang}]{zhu2021least}
Zhu, X., Li, F., and Wang, H. (2021), \enquote{Least-Square Approximation for a
  Distributed System,} \textit{Journal of Computational and Graphical
  Statistics}, 1--15.

\end{thebibliography}

\end{document}